\documentclass{article}

\usepackage[final,nonatbib]{neurips_2020}

\usepackage[utf8]{inputenc} 
\usepackage[T1]{fontenc}    
\usepackage{hyperref}       
\usepackage{url,enumitem}            
\usepackage{booktabs}       
\usepackage{amsfonts}       
\usepackage{nicefrac}       
\usepackage{microtype}   
\usepackage{xcolor}
\usepackage{microtype}
\usepackage{graphicx}
\usepackage{subfigure}
\usepackage{booktabs} 
\usepackage{thm-restate}
\usepackage{multirow}
\usepackage{amsfonts}
\usepackage{amsmath,amsfonts,bm}
\usepackage{amssymb,amsmath}
\usepackage{amsthm}
\usepackage[scr=boondoxo]{mathalfa}

\newcommand{\bfem}[1]{\textbf{\emph{#1}}}

\newtheorem{definition}{Definition}
\newtheorem{assumption}{Assumption}
\newtheorem{lemma}{Lemma}

\title{On the Modularity of Hypernetworks}

\author{%
Tomer Galanti \\
School of Computer Science \\
Tel Aviv University \\
\texttt{tomerga2@tauex.tau.ac.il} \\
\And
Lior Wolf \\
Facebook AI Research (FAIR) \& \\
Tel Aviv University \\
\texttt{wolf@fb.com} 
}

\begin{document}

\maketitle

\begin{abstract}
In the context of learning to map an input $I$ to a function $h_I:\mathcal{X}\to \mathbb{R}$, two alternative methods are compared: (i) an embedding-based method, which learns a fixed function in which $I$ is encoded as a conditioning signal $e(I)$ and the learned function takes the form $h_I(x) = q(x,e(I))$, and (ii) hypernetworks, in which the weights $\theta_I$ of the function $h_I(x) = g(x;\theta_I)$ are given by a hypernetwork $f$ as $\theta_I=f(I)$. 
In this paper, we define the property of modularity as the ability to effectively learn a different function for each input instance $I$. For this purpose, we adopt an expressivity perspective of this property and extend the theory of~\cite{devore} and provide a lower bound on the complexity (number of trainable parameters) of neural networks as function approximators, by eliminating the requirements for the approximation method to be robust. Our results are then used to compare the complexities of $q$ and $g$, showing that under certain conditions and when letting the functions $e$ and $f$ be as large as we wish, $g$ can be smaller than $q$ by orders of magnitude. This sheds light on the modularity of hypernetworks in comparison with the embedding-based method. Besides, we show that for a structured target function, the overall number of trainable parameters in a hypernetwork is smaller by orders of magnitude than the number of trainable parameters of a standard neural network and an embedding method. 
\end{abstract}


\section{Introduction}

Conditioning refers to the existence of multiple input signals. For example, in an autoregressive model, where the primary input is the current hidden state or the output of the previous time step, a conditioning signal can drive the process in the desired direction. When performing text to speech with WaveNets~\cite{vanwavenet}, the autoregressive signal is concatenated to the conditioning signal arising from the language features. Other forms of conditioning are less intuitive. For example, in Style GANs~\cite{karras2019style}, conditioning takes place by changing the weights of the normalization layers according to the desired style.


In various settings, it is natural to treat the two inputs $x$ and $I$ of the target function $y(x,I)$ as nested, i.e., multiple inputs $x$ correspond to the `context' of the same conditioning input $I$. A natural modeling~\cite{Chen_2019_CVPR,Park_2019_CVPR,Mescheder_2019_CVPR} is to encode the latter by some embedding network $e$ and to concatenate it to $x$ when performing inference $q(x,e(I))$ with a primary network $q$. A less intuitive solution, commonly referred to as a hypernetwork, uses a primary network $g$ whose weights are not directly learned. Instead, $g$ has a fixed architecture, and a second network $f$ generates its weights based on the conditioning input as $\theta_I = f(I)$. The network $g$, with the weights $\theta_I$ can then be applied to any input $x$.

Hypernetworks hold state of the art results on numerous popular benchmarks~\cite{bertinetto2016learning,Oswald2020Continual,brock2018smash,zhang2018graph,lorraine2018stochastic}, especially due to their ability to adapt $g$ for different inputs $I$. This allows the hypernetwork to model tasks effectively, even when using a low-capacity $g$.  
This lack of capacity is offset by using very large networks $f$. For instance, in~\cite{Littwin_2019_ICCV}, a deep residual hypernetwork that is trained from scratch outperforms numerous embedding-based networks that rely on ResNets that were pre-trained on ImageNet. 

The property of modularity means that through $f$, the network $g$ is efficiently parameterized. Consider the case in which we fit individual functions $g'_I$ to model each function $y_I = y(\cdot,I)$ independently (for any fixed $I$). To successfully fit any of these functions, $g'_I$ would require some degree of minimal complexity in the worst case. We say that modularity holds if the primary network $g$ whose weights are given by $f(I)$ has the same minimal complexity as required by $g'_I$ in the worst case.

In this paper, we seek to understand this phenomenon. For this purpose, we compare two alternatives: the  standard embedding method and the hypernetwork. Since neural networks often have millions of weights while embedding vectors have a dimension that is seldom larger than a few thousand, it may seem that $f$ is much more complex than $e$. However, in hypernetworks, often the output of $f$ is simply a linear projection of a much lower dimensional bottleneck~\cite{Littwin_2019_ICCV}. More importantly, it is often the case that the function $g$ can be small, and it is the adaptive nature (where $g$ changes according to $I$) that enables the entire hypernetwork ($f$ and $g$ together) to be expressive.


In general, the formulation of hypernetworks covers embedding-based methods. This implies that hypernetworks are at least as good as the embedding-based method and motivates the study of whether hypernetworks have a clear and measurable advantage. Complexity analysis provides a coherent framework to compare the two alternatives. In this paper, we compare the minimal parameter complexity needed to obtain a certain error in each of the two alternatives.

%

\noindent{\bf Contributions\quad} The central contributions of this paper are: 
    {\bf (i)} Thm.~\ref{thm:generallower} extends the theory of~\cite{devore} and provides a lower bound on the number of trainable parameters of a neural network when approximating smooth functions. In contrast to previous work, our result does not require that the approximation method is robust. 
    {\bf (ii)} In Thms.~\ref{thm:2}-\ref{thm:4}, we compare the complexities of the primary functions under the two methods ($q$ and $g$) and show that for a large enough embedding function, the hypernetwork's primary $g$ can be smaller than $q$ by orders of magnitude. 
    {\bf (iii)} In Thm.~\ref{thm:5}, we show that under common assumptions on the function to be approximated, the overall number of trainable parameters in a hypernetwork is much smaller than the number of trainable parameters of a standard neural network. {\bf (iv)} To validate the theoretical observations, we conducted experiments on synthetic data as well as on self-supervised learning tasks. 
    
To summarize, since Thm.~\ref{thm:generallower} shows the minimal complexity for approximating smooth target functions, and Thm.~\ref{thm:4} demonstrates that this is attainable by a hypernetwork, we conclude that hypernetworks are modular. In contrast, embedding methods are not since, as Thms.~\ref{thm:2}-\ref{thm:3} show, they require a significantly larger primary.

\noindent{\bf Related Work \quad} Hypernetworks, which were first introduced under this name in~\cite{ha2016hypernetworks}, are networks that generate the weights of a second {\em primary} network that computes the actual task. The Bayesian formulation of \cite{krueger2017bayes} introduces variational inference that involves both the parameter generating network and a primary network. Hypernetworks are especially suited for meta-learning tasks, such as few-shot~\cite{bertinetto2016learning} and continual learning tasks~\cite{Oswald2020Continual}, due to the knowledge sharing ability of the weights generating network. 
Predicting the weights instead of performing backpropagation can lead to efficient neural architecture search~\cite{brock2018smash,zhang2018graph}, and hyperparameter selection~\cite{lorraine2018stochastic}. 

Multiplicative interactions, such as gating, attention layers, hypernetworks, and dynamic convolutions, were shown to strictly extend standard neural networks~\cite{jayakumar2020multiplicative}. However, the current literature has no theoretical guarantees that support the claim that interactions have a clear advantage. 


In this work, we take an approximation theory perspective of this problem. For this purpose, as a starting point, we study standard neural networks as function approximators. There were various attempts to understand the capabilities of neural networks as universal function approximators~\cite{Cybenko1989,Hornik1991ApproximationCO}. 
Multiple extensions of these results~\cite{Mhaskar:1996:NNO:1362203.1362213,NIPS2017_7203,hanin2017approximating,10.5555/3327345.3327515,pmlr-v70-safran17a} quantify tradeoffs between the number of trainable parameters, width and depth of the neural networks as universal approximators. In particular,~\cite{Mhaskar:1996:NNO:1362203.1362213} suggested upper bounds on the size of the neural networks of order $\mathcal{O}(\epsilon^{-n/r})$, where $n$ is the input dimension, $r$ is the order of smoothness of the target functions, and $\epsilon>0$ is the approximation accuracy. In another contribution, \cite{devore} prove a lower bound on the complexity of the class of approximators that matches the upper bound $\Omega(\epsilon^{-n/r})$. However, their analysis assumes that the approximation is robust in some sense (see Sec.~\ref{sec:analysis} for details). In addition, they show that robustness holds when the class of approximators $\mathscr{f} = \{f(\cdot;\theta) \mid \theta \in \Theta_{\mathscr{f}}\}$ satisfies a certain notion of bi-Lipschitzness. However, as a consequence of this condition, any two equivalent functions (i.e., $f(\cdot;\theta_1) = f(\cdot;\theta_2)$) must share the same parameterizations (i.e., $\theta_1=\theta_2$). Unfortunately, this condition is not met for neural networks, as one can compute the same function with neural networks of the same architecture with different parameterizations. In Sec.~\ref{sec:analysis}, we show that for certain activation functions and under reasonable conditions, there exists a robust approximator and, therefore, the lower bound of the complexity is $\Omega(\epsilon^{-n/r})$.  Since the existence of a continuous selector is also a cornerstone in the proofs of Thms.~\ref{thm:2}-\ref{thm:5}, the analysis in~\cite{devore} is insufficient to prove these results (for example, see the proof sketch of Thm.~\ref{thm:4} in Sec.~\ref{sec:main1}). In~\cite{10.1006/jath.1998.3304,10.1006/jath.1998.3305,Maiorov99lowerbounds} a similar lower bound is shown, but, only for shallow networks.

In an attempt to understand the benefits of locality in convolutional neural networks,~\cite{10.5555/3298483.3298577} shows that when the target function is a hierarchical function,  it can be approximated by a hierarchic neural network of smaller complexity, compared to the worst-case complexity for approximating arbitrary functions. In our Thm.~\ref{thm:5}, we take a similar approach. We show that under standard assumptions in meta-learning, the overall number of trainable parameters in a hypernetwork necessary to approximate the target function is smaller by orders of magnitude, compared to approximating arbitrary functions with neural networks and the embedding method in particular.


\section{Problem Setup}\label{sec:setup}

In various meta-learning settings, we have an unknown target function $y:\mathcal{X} \times \mathcal{I} \to \mathbb{R}$ that we would like to model. Here, $x \in \mathcal{X}$ and $I \in \mathcal{I}$ are two different inputs of $y$. The two inputs have different roles, as the input $I$ is ``task'' specific and $x$ is independent of the task. Typically, the modeling of $y$ is done in the following manner: $H(x,I) = G(x,E(I)) \approx y(x,I)$, where $E$ is an embedding function and $G$ is a predictor on top of it. The distinction between different embedding methods stems from the architectural relationship between $E$ and $G$. In this work, we compare two task embedding methods: (i) neural embedding methods and (ii) hypernetworks.

A {\bf neural embedding method} is a network of the form $h(x,I;\theta_e,\theta_q) = q(x,e(I;\theta_e);\theta_q) $, consisting of a composition of neural networks $q$ and $e$ parameterized with real-valued vectors $\theta_q \in \Theta_{\mathscr{q}}$ and $\theta_e \in \Theta_{\mathscr{e}}$ (resp.). The term $e(I;\theta_e)$ serves as an embedding of $I$. For two given families $\mathscr{q} := \{q(x,z;\theta_q) \;\vert\; \theta_q \in \Theta_{\mathscr{q}}\}$ and $\mathscr{e} := \{e(I;\theta_e) \;\vert\; \theta_e \in \Theta_{\mathscr{e}}\}$ of functions, we denote by $\mathcal{E}_{\mathscr{e},\mathscr{q}} := \{q(x,e(I;\theta_e);\theta_q) \;\vert\; \theta_q \in \Theta_{\mathscr{q}}, \theta_e \in \Theta_{\mathscr{e}}\}$ the embedding method that is formed by them.  

A special case of neural embedding methods is the family of the conditional neural processes models~\cite{garnelo2018conditional}. In such processes, $\mathcal{I}$ consists of a set of $d$ images $I = (I_i)^{d}_{i=1} \in \mathcal{I}$, and the embedding is computed as an average of the embeddings over the batch, $e(I;\theta_e) := \frac{1}{d} \sum^{d}_{i=1} e(I_i;\theta_e)$. 

A {\bf hypernetwork} $h(x,I) = g(x;f(I;\theta_f))$ is a pair of collaborating neural networks, $f:\mathcal{I} \to \Theta_{\mathscr{g}}$ and $g:\mathcal{X} \to \mathbb{R}$, such that for an input $I$, $f$ produces the weights $\theta_I = f(I;\theta_f)$ of $g$, where $\theta_f \in \Theta_{\mathscr{f}}$ consists of the weights of $f$. The function $f(I;\theta_f)$ takes a conditioning input $I$ and returns the parameters $\theta_I \in \Theta_{\mathscr{g}}$ for $g$. The network $g$ takes an input $x$ and returns an output $g(x;\theta_I)$ that depends on both $x$ and the task specific input $I$. In practice, $f$ is typically a large neural network and $g$ is a small neural network. 

The entire prediction process for hypernetworks is denoted by $h(x,I;\theta_f)$, and the set of functions $h(x,I;\theta_f)$ that are formed by two families $\mathscr{f} := \{f(I;\theta_f) \;\vert\; \theta_f \in \Theta_{\mathscr{f}}\}$ and $\mathscr{g} := \{g(x;\theta_g) \;\vert\; \theta_g \in \Theta_{\mathscr{g}}\}$ as a hypernetwork is denoted by $\mathcal{H}_{\mathscr{f},\mathscr{g}} := \{g(x;f(I;\theta_f)) \;\vert\; \theta_f \in \Theta_{\mathscr{f}}\}$.

    


\subsection{Terminology and Notations} 

We consider $\mathcal{X} = [-1,1]^{m_1}$ and $\mathcal{I} = [-1,1]^{m_2}$ and denote, $m := m_1+m_2$. For a closed set $X\subset \mathbb{R}^n$, we denote by $C^r(X)$ the linear space of all $r$-continuously differentiable functions $h: X \to \mathbb{R}$ on $X$ equipped with the supremum norm $\|h\|_{\infty} := \max_{x \in X} \| h(x) \|_1$. 
We denote parametric classes of functions by calligraphic lower letters, e.g., $\mathscr{f} = \{f(\cdot ;\theta_f) :\mathbb{R}^m \to \mathbb{R}\;\vert\; \theta_f \in \Theta_{\mathscr{f}} \}$. A specific function from the class is denoted by the non-calligraphic lower case version of the letter $f$ or $f(x;\theta_f)$. The notation ``$;$'' separates between direct inputs of the function $f$ and its parameters $\theta_f$. Frequently, we will use the notation $f(\cdot;\theta_f)$, to specify a function $f$ and its parameters $\theta_f$ without specifying a concrete input of this function. 
The set $\Theta_{\mathscr{f}}$ is closed a subset of $\mathbb{R}^{N_{\mathscr{f}}}$ and consists of the various parameterizations of members of $\mathscr{f}$ and $N_{\mathscr{f}}$ is the number of parameters in $\mathscr{f}$, referred to as the complexity of $\mathscr{f}$. 

A class of neural networks $\mathscr{f}$ is a set of functions of the form: 
\begin{equation}\label{eq:net}
f(x;[\bfem{W},\bfem{b}]) := W^{k} \cdot \sigma(W^{k-1} \dots \sigma(W^1x+b^1) + b^{k-1}) 
\end{equation}
with weights $W^i \in \mathbb{R}^{h_{i+1} \times h_{i}}$ and biases $b^i \in \mathbb{R}^{h_{i+1}}$, for some $h_i \in \mathbb{N}$. In addition, $\theta := [\bfem{W},\bfem{b}]$ accumulates the parameters of the network. 
The function $\sigma$ is a non-linear activation function, typically ReLU, logistic function, or the hyperbolic tangent.

We define the spectral complexity of a network $f := f(\cdot;[\bfem{W},\bfem{b}])$ as  $\mathcal{C}(f) := \mathcal{C}([\bfem{W},\bfem{b}]) := L^{k-1}\cdot \prod^{k}_{i=1} \|W^i\|_1$, where $\|W\|_{1}$ is the induced $L_1$ matrix norm 
and $L$ is the Lipschitz constant of $\sigma$. In general, $\mathcal{C}(f)$ upper bounds the Lipschitz constant of $f$ (see Lem.~3 in the appendix).

Throughout the paper, we consider the Sobolev space $\mathcal{W}_{r,n}$ as the set of target functions to be approximated. This class consists of $r$-smooth functions of bounded derivatives. Formally, it consists of functions $h:[-1,1]^n \to \mathbb{R}$ with continuous partial derivatives of orders up to $r$, such that, the Sobolev norm is bounded, $\|h\|^{s}_{r} := \|h\|_{\infty} + \sum_{1 \leq \vert \mathbf{k}\vert_1 \leq r} \|D^{\mathbf{k}} h\|_{\infty}\leq 1$, where $D^{\mathbf{k}}$ denotes the partial derivative indicated by the multi–integer $\mathbf{k} \geq 1$, and $\vert \mathbf{k} \vert_1$ is the sum of the components of $\mathbf{k}$. Members of this class are typically the objective of approximation in the literature~\cite{Mhaskar:1996:NNO:1362203.1362213,10.1006/jath.1998.3304,NIPS2017_7203}. 

In addition, we define the class $\mathcal{P}^{k_1,k_2}_{r,w,c}$ to be the set of functions $h:\mathbb{R}^{k_1} \to \mathbb{R}^{k_2}$ of the form $h(x) = W \cdot P(x)$, where $P:\mathbb{R}^{k_1} \to \mathbb{R}^{w}$ and $W \in \mathbb{R}^{k_2 \times w}$ is some matrix of the bounded induced $L_1$ norm $\|W\|_{1} \leq c$. Each output coordinate $P_{i}$ of $P$ is a member of $\mathcal{W}_{r,k_1}$. The linear transformation on top of these functions serves to enable blowing up the dimension of the produced output. However, the ``effective'' dimensionality of the output is bounded by $w$. For simplicity, when $k_1$ and $k_2$ are clear from context, we simply denote $\mathcal{P}_{r,w,c} := \mathcal{P}^{k_1,k_2}_{r,w,c}$. We can think of the functions in this set as linear projections of a set of features of size $w$.

\noindent{\bf Assumptions\quad} Several assumptions were made to obtain the theoretical results. The first one is not strictly necessary, but significantly reduces the complexity of the proofs: we assume the existence of a unique function $f \in \mathscr{f}$ that best approximates a given target function $y$. It is validated empirically in Sec.~\ref{sec:experiments}.

\begin{assumption}[Unique Approximation]\label{assmp:apprx} 
Let $\mathscr{f}$ be a class of neural networks. Then, for all $y \in \mathbb{Y}$ there is a unique function $f(\cdot;\theta^*) \in \mathscr{f}$ that satisfies: $\|f(\cdot;\theta^*) - y\|_{\infty} = \inf_{\theta \in \Theta_{\mathscr{f}}} \| f(\cdot;\theta) - y\|_{\infty}$. 
\end{assumption}






For simplicity, we also assume that the parameters $\theta^*$ of the best approximators are bounded (uniformly, for all $y\in \mathbb{Y}$). The next assumption is intuitive and asserts that for any target function $y$ that is being approximated by a class of neural networks $\mathscr{f}$, by adding a neuron to the architecture, one can achieve a strictly better approximation to $y$ or $y$ is already perfectly approximated by $\mathscr{f}$. 

\begin{assumption}\label{assmp:improve}
Let $\mathscr{f}$ be a class of neural networks. Let $y \in \mathbb{Y}$ be some function to be approximated. Let $\mathscr{f}'$ be a class of neural networks that resulted by adding a neuron to some hidden layer of $\mathscr{f}$. If $y \notin \mathscr{f}$ then, 
$\inf_{\theta \in \Theta_{\mathscr{f}}} \| f(\cdot;\theta) - y\|_{\infty} > \inf_{\theta \in \Theta_{\mathscr{f}'}} \| f(\cdot;\theta) - y\|_{\infty}$.
\end{assumption}

This assumption is validated empirically in Sec.~1.5 of the appendix. In the following lemma, we prove that Assumption~\ref{assmp:improve} holds for shallow networks for the $L_2$ distance instead of $L_{\infty}$.
\begin{lemma}
Let $\mathbb{Y} = C([-1,1]^m)$ be the class of continuous functions $y:[-1,1]^m \to \mathbb{R}$. Let $\mathscr{f}$ be a class of 2-layered neural networks of width $d$ with $\sigma$ activations, where $\sigma$ is either $\tanh$ or sigmoid. Let $y \in \mathbb{Y}$ be some function to be approximated. Let $\mathscr{f}'$ be a class of neural networks that is resulted by adding a neuron to the hidden layer of $\mathscr{f}$. If $y \notin \mathscr{f}$ then, $\inf_{\theta \in \Theta_{\mathscr{f}}} \| f(\cdot;\theta) - y\|^2_{2} > \inf_{\theta \in \Theta_{\mathscr{f}'}} \| f(\cdot;\theta) - y\|^2_{2}$. The same holds for $\sigma = ReLU$ when $m=1$.
\end{lemma}



\section{Degrees of Approximation}\label{sec:analysis}

We are interested in determining how complex a model ought to be to theoretically guarantee approximation of an unknown target function $y$ up to a given approximation error $\epsilon > 0$. Formally, let $\mathbb{Y}$ be a set of target functions to be approximated. For a set $\mathcal{P}$ of candidate approximators, we measure its ability to approximate $\mathbb{Y}$ as: $d(\mathcal{P};\mathbb{Y}) := \sup_{y \in \mathbb{Y}}\inf_{p \in \mathcal{P}} \|y - p \|_{\infty}$.
This quantity measures the maximal approximation error for approximating a target function $y \in \mathbb{Y}$ using candidates $p$ from $\mathcal{P}$.

Typical approximation results show that the class $\mathbb{Y} = \mathcal{W}_{r,m}$ can be approximated using classes of neural networks $\mathscr{f}$ of sizes $\mathcal{O}(\epsilon^{-m/r})$, where $\epsilon$ is an upper bound on $d(\mathscr{f};\mathbb{Y})$. For instance, in~\cite{Mhaskar:1996:NNO:1362203.1362213} this property is shown for neural networks with activations $\sigma$ that are infinitely differentiable and not polynomial on any interval; ~\cite{hanin2017approximating} prove this property for ReLU neural networks. We call activation functions with this property {\em universal}.

\begin{definition}[Universal activation] An activation function $\sigma$ is universal if for any $r,n \in \mathbb{N}$ and $\epsilon > 0$, there is a class of neural networks $\mathscr{f}$ with $\sigma$ activations, of size $\mathcal{O}(\epsilon^{-n/r})$, such that, $d(\mathscr{f};\mathcal{W}_{r,n})\leq \epsilon$.
\end{definition}

An interesting question is whether this bound is tight. 
We recall the $N$-width framework of~\cite{devore} (see also~\cite{Pinkus1985NWidthsIA}). Let $\mathscr{f}$ be a class of functions (not necessarily neural networks) and $S: \mathbb{Y} \to \mathbb{R}^{N}$ be a continuous mapping between a function $y$ and its approximation, where with $N := N_{\mathscr{f}}$. In this setting, we approximate $y$ using $f(\cdot;S(y))$, where the continuity of $S$ means that the selection of parameters is robust with respect to perturbations in $y$. The nonlinear $N$-width of the compact set $\mathbb{Y} = \mathcal{W}_{r,m}$ is defined as follows:
\begin{equation}
\begin{aligned}
\tilde{d}_N(\mathbb{Y}) := \inf_{\mathscr{f}} \tilde{d}(\mathscr{f};\mathbb{Y})
:= \inf_{\mathscr{f}} \inf_{S} \sup_{y \in \mathbb{Y}} \| f(\cdot;S(y)) - y\|_{\infty},
\end{aligned}
\end{equation}
where the infimum is taken over classes $\mathscr{f}$, such that, $N_{\mathscr{f}} = N$ and $S$ is continuous. {Informally, the $N$-width of the class $\mathbb{Y}$ measures the minimal approximation error achievable by a continuous function $S$ that selects approximators $f(\cdot;S(y))$ for the functions $y \in \mathbb{Y}$.} As shown by~\cite{devore}, $\tilde{d}_N(\mathbb{Y}) = \Omega(N^{-m/r})$, or alternatively, if there exists $\mathscr{f}$, such that, $\tilde{d}(\mathscr{f};\mathbb{Y})\leq \epsilon$ (i.e., $\tilde{d}_{N_{\mathscr{f}}}(\mathbb{Y}) \leq \epsilon$), then $N_{\mathscr{f}} = \Omega(\epsilon^{-m/r})$. We note that since the $N$-width of $\mathbb{Y}$ is oblivious of the class of approximators $\mathscr{f}$ and $d(\mathscr{f};\mathbb{Y}) \leq \tilde{d}(\mathscr{f};\mathbb{Y})$ and, therefore, this analysis does not provide a full solution to this question. Specifically, to answer this question, it requires a nuanced treatment of the considered class of approximators $\mathscr{f}$.

In the following theorem, we show that under certain conditions, the lower bound holds, even when removing the assumption that the selection is robust. 

\begin{restatable}{theorem}{generallower}\label{thm:generallower}
Let $\sigma$ be a piece-wise $C^1(\mathbb{R})$ activation function with $\sigma' \in BV(\mathbb{R})$. Let $\mathscr{f}$ be a class of neural networks with $\sigma$ activations. Let $\mathbb{Y} = \mathcal{W}_{r,m}$. Assume that any non-constant $y \in \mathbb{Y}$ is not a member of $\mathscr{f}$. Then, if $d(\mathscr{f};\mathbb{Y}) \leq \epsilon$, we have $N_{\mathscr{f}} = \Omega(\epsilon^{-m/r})$.
\end{restatable}
All of the proofs are provided in the appendix. The notation $BV(\mathbb{R})$ stands for the set of functions of bounded variation, 
\begin{equation}
BV(\mathbb{R}) := \left\{ f \in L^1(\mathbb{R}) \mid \| f\|_{BV}  < \infty \right\} \textnormal{ where, } \| f\|_{BV} := \sup\limits_{\substack{\phi \in C^1_c(\mathbb{R}) \\ \|\phi\|_{\infty}\leq 1}} 
\int_{\mathbb{R}} f(x)\cdot \phi(x)\;\textnormal{d}x
\end{equation}
We note that a wide variety of activation functions satisfy the conditions of Thm.~\ref{thm:generallower}, such as, the clipped ReLU, sigmoid, $\tanh$ and $\arctan$. Informally, to prove this theorem, we show the existence of a ``wide'' subclass $\mathbb{Y}' \subset \mathbb{Y}$ and a {\bf continuous selector} $S:\mathbb{Y}' \to \Theta_{\mathscr{f}}$, such that, $\exists\alpha > 0~\forall y\in \mathbb{Y}':\|f(\cdot;S(y)) - y\|_{\infty} \leq \alpha \cdot \inf_{\theta \in \Theta_{\mathscr{f}}} \|f(\cdot;\theta) - y\|_{\infty}$. The class $\mathbb{Y}'$ is considered wide in terms of $N$-width, i.e., $\tilde{d}_N(\mathbb{Y}') = \Omega(N^{-m/r})$. Therefore, we conclude that $d(\mathscr{f};\mathbb{Y}) \geq d(\mathscr{f};\mathbb{Y}') \geq \frac{1}{\alpha} \tilde{d}(\mathscr{f};\mathbb{Y}') = \Omega(N^{-m/r})$. For further details, see the proof sketches in Secs.~3.2-3.3 of the appendix. Finally, we note that the assumption that any non-constant $y \in \mathbb{Y}$ is not a member of $\mathscr{f}$ is rather technical. For a relaxed, for general version of it, see Lem.~18 in the appendix.

%

\section{Expressivity of Hypernetworks}

Using Thm.~\ref{thm:generallower}, the expressive power of hypernetworks is demonstrated. In the first part, we compare the complexities of $\mathscr{g}$ and $\mathscr{q}$. We show that when letting $\mathscr{e}$ and $\mathscr{f}$ be large enough, one can approximate it using a hypernetwork where $\mathscr{g}$ is smaller than $\mathscr{q}$ by orders of magnitude. In the second part, we show that under typical assumptions on $y$, one can approximate $y$ using a hypernetwork with overall much fewer parameters than the number of parameters required for a neural embedding method. It is worth mentioning that our results scale to the multi-dimensional case. In this case, if the output dimension is constant, we get the exact same bounds.

\subsection{Comparing the complexities of $\mathscr{q}$ and $\mathscr{g}$}\label{sec:main1}
We recall that for an arbitrary $r$-smooth function $y\in \mathcal{W}_{r,n}$, the complexity for approximating it is $\mathcal{O}(\epsilon^{-n/r})$. We show that hypernetwork models can effectively learn a different function for each input instance $I$. Specifically, a hypernetwork is able to capture a separate approximator $h_I = g(\cdot;f(I;\theta_f))$ for each $y_I$ that has a minimal complexity $\mathcal{O}(\epsilon^{-m_1/r})$. On the other hand, we show that for a smoothness order of $r=1$, under certain constraints, when applying an embedding method, it is impossible to provide a separate approximator $h_I = q(\cdot,e(I;\theta_e);\theta_q)$ of complexity $\mathcal{O}(\epsilon^{-m_1})$. Therefore, the embedding method does not enjoy the same modular properties of hypernetworks.



The following result demonstrates that the complexity of the main-network $q$ in any embedding method has to be of non-optimal complexity. As we show, it holds regardless of the size of $\mathscr{e}$, as long as the functions $e \in \mathscr{e}$ are of bounded Lipschitzness. 

\begin{restatable}{theorem}{two}\label{thm:2}
Let $\sigma$ be a universal, piece-wise $C^1(\mathbb{R})$ activation function with $\sigma' \in BV(\mathbb{R})$ and $\sigma(0)=0$. Let $\mathcal{E}_{\mathscr{e},\mathscr{q}}$ be a neural embedding method. Assume that $\mathscr{e}$ is a class of continuously differentiable neural network $e$ with zero biases, output dimension $k = \mathcal{O}(1)$ and 
$\mathcal{C}(e) \leq \ell_1$ and $\mathscr{q}$ is a class of neural networks $q$ with $\sigma$ activations and 
$\mathcal{C}(q) \leq \ell_2$. Let $\mathbb{Y} := \mathcal{W}_{1,m}$. Assume that any non-constant $y\in \mathbb{Y}$ cannot be represented as a neural network with $\sigma$ activations. If the embedding method achieves error $d(\mathcal{E}_{\mathscr{e},\mathscr{q}},\mathbb{Y}) \leq \epsilon$, then, the complexity of $\mathscr{q}$ is: $N_{\mathscr{q}} = \Omega\left(\epsilon^{-(m_1+m_2)}\right)$.
\end{restatable}

\begin{figure}
  \begin{minipage}[c]{0.60\textwidth}
  \begin{tabular}{@{}c@{}c@{}}
     \includegraphics[width=0.5\linewidth]{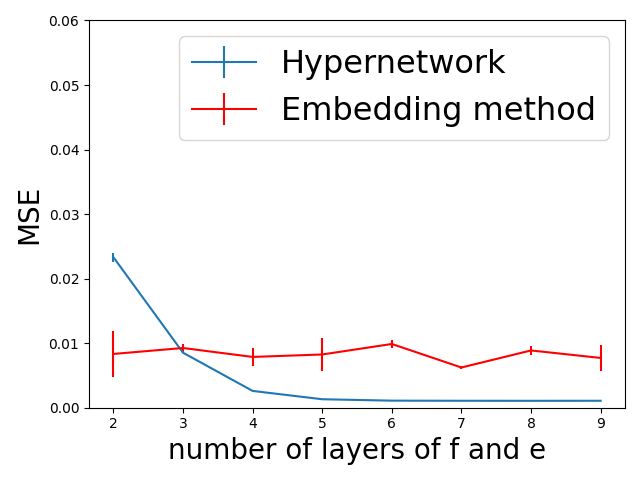} & \includegraphics[width=0.5\linewidth]{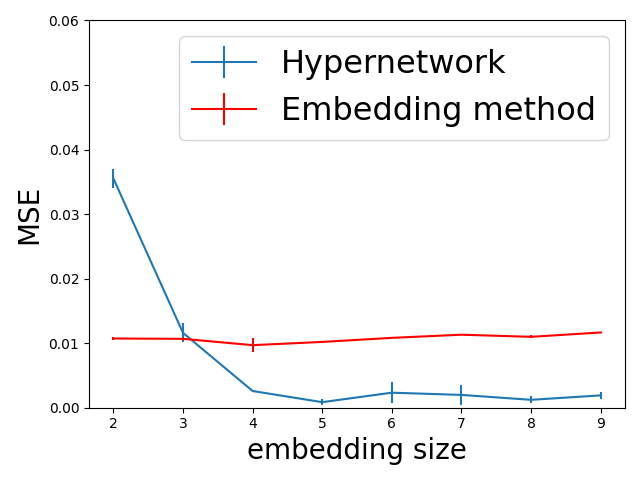} \\
     (a)&(b)\\
  \end{tabular}
 \end{minipage}%
 \hfill
  \begin{minipage}[c]{0.39\textwidth}
    \caption{(a) MSE error obtained by hypernetworks and the embedding method with varying number of layers (x-axis). Synthetic target functions $y(x,I) = \langle x,h(I)\rangle$, for neural network $h$. 
    (b) Varying the embedding layer to be 100/1000 (depending on the method) times the x-axis. error bars are SD over 100 repetitions.}
    \label{fig:comparison}
  \end{minipage}
\end{figure}

The following theorem extends Thm.~\ref{thm:2} to the case where the output dimension of $e$ depends on $\epsilon$. In this case, the parameter complexity is also non-optimal.

\begin{restatable}{theorem}{three}\label{thm:3}
In the setting of Thm.~\ref{thm:2}, except $k$ is not necessarily $\mathcal{O}(1)$. Assume that the first layer of any $q \in \mathscr{q}$ is bounded $\|W^1\|_1 \leq c$, for some constant $c>0$. If the embedding method achieves error $d(\mathcal{E}_{\mathscr{e},\mathscr{q}},\mathbb{Y}) \leq \epsilon$, then, the complexity of $\mathscr{q}$ is: $N_{\mathscr{q}} = \Omega\left(\epsilon^{-\min(m,2m_1) }\right)$.
\end{restatable}

{The results in Thms.~\ref{thm:2}-\ref{thm:3} are limited to $r=1$, which in the context of the Sobolev space $r=1$ means bounded, Lipschitz and continuously differentiable functions. The ability to approximate these functions is studied extensively in the literature~\cite{NIPS2017_7203,hanin2017approximating}. Extending the results for $r>1$ is possible but necessitates the introduction of spectral complexities that correspond to higher-order derivatives. }


The following theorem shows that for any function $y \in \mathcal{W}_{r,m}$, there is a large enough hypernetwork, that maps between $I$ and an approximator of $y_I$ of optimal complexity.



\begin{restatable}{theorem}{four}[Modularity of Hypernetworks]\label{thm:4}
Let $\sigma$ be as in Thm.~\ref{thm:2}. Let $y \in \mathbb{Y} = \mathcal{W}_{r,m}$ be a function, such that, $y_I$ cannot be represented as a neural network with $\sigma$ activations for all $I \in \mathcal{I}$. Then, there is a class, $\mathscr{g}$, of neural networks with $\sigma$ activations and a network $f(I;\theta_f)$ with ReLU activations, such that, $h(x,I) = g(x;f(I;\theta_f))$ achieves error $\leq \epsilon$ in approximating $y$ and $N_{\mathscr{g}} = \mathcal{O}\left(\epsilon^{-m_1/r}\right)$.
\end{restatable}

{Recall that Thm.~\ref{thm:generallower} shows that the minimal complexity for approximating each  individual smooth target function $y_I$ is $\mathcal{O}(\epsilon^{-m_1/r})$. Besides, Thm.~\ref{thm:4} shows that this level of fitting is attainable by a hypernetwork for all $y_I$. Therefore, we conclude that hypernetworks are modular. On the other hand, from Thms.~\ref{thm:2}-\ref{thm:3} we conclude that this is not the case for the embedding method.}

When comparing the results in Thms.~\ref{thm:2},~\ref{thm:3} and~\ref{thm:4} in the case of $r=1$, we notice that in the hypernetworks case, $\mathscr{g}$ can be of complexity $\mathcal{O}(\epsilon^{-m_1})$ in order to achieve approximation error $\leq \epsilon$. On the other hand, for the embedding method case, the complexity of the primary-network $q$ is at least $\Omega(\epsilon^{-(m_1+m_2)})$ when the embedding dimension is of constant size and at least $\Omega\left(\epsilon^{-\min(m,2m_1) }\right)$ when it is unbounded to achieve approximation error $\leq \epsilon$. In both cases, the primary network of the embedding method is larger by orders of magnitude than the primary network of the hypernetwork.

Note that the embedding method can be viewed as a simple hypernetwork, where only the biases of the first layer of $g$ are given by $f$. Therefore, the above results show that the modular property of hypernetworks, which enables $g$ to be of small complexity, emerges only when letting $f$ produce the whole set of weights of $g$. We note that this kind of emulation is not symmetric, as it is impossible to emulate a hypernetwork with the embedding method as it is bound to a specific structure defined by $q$ being a neural network (as in Eq.~\ref{eq:net}) that takes the concatenation of $x$ and $e(I;\theta_e)$ as its input.

{{\bf Proof sketch of Thm.~\ref{thm:4}\quad} 
Informally, the theorem follows from three main arguments: {\bf(i)} we treat $y(x,I)$ as a class of functions $\mathcal{Y}:=\{y_I\}_{I\in \mathcal{I}}$ and take a class $\mathscr{g}$ of neural networks of size $\mathcal{O}(\epsilon^{-m_2/r})$, that achieves $d(\mathscr{g};\mathcal{Y})\leq \epsilon$,
{\bf(ii)} we prove the existence of a continuous selector for $\mathcal{Y}$ within $\mathscr{g}$ and {\bf (iii)} we draw a correspondence between the continuous selector and modeling $y$ using a hypernetwork. 

We want to show the existence of a class $\mathscr{g}$ of size $\mathcal{O}(\epsilon^{-m_2/r})$ and a network $f(I;\theta_f)$, such that,
\begin{equation}
\sup_{I}\|g(\cdot ;f(I;\theta_f)) - y_I\|_{\infty} \leq 3 \sup_{I}\inf_{\theta_g}\|g(\cdot ;\theta_g) - y_I\|_{\infty} \leq 3\epsilon
\end{equation}
We note that this expression is very similar to a robust approximation of the class $\mathcal{Y}$, except the selector $S(y_I)$ is replaced with a network $f(I;\theta_f)$. Since $\sigma$ is universal, there exists an architecture $\mathscr{g}$ of size $\mathcal{O}(\epsilon^{-m_1/r})$, such that,  $d(\mathscr{g};\mathcal{Y})  \leq \epsilon$. In addition, we prove the existence of a continuous selector $S:\mathcal{Y} \to \Theta_{\mathscr{g}}$, i.e., $\sup_I\| g(\cdot;S(y_I))-y_I\|_{\infty} \leq 2d(\mathscr{g};\mathcal{Y}) \leq 2\epsilon$. 

As a next step, we replace $S$ with a neural network $f(I;\theta_f)$. Since $I \mapsto y_I$ is a continuous function, the function $\hat{S}(I) := S(y_I)$ is continuous as well. Furthermore, as we show, $g$ is uniformly continuous with respect to both $x$ and $\theta_g$, and therefore, by ensuring that $\inf_{\theta_f} \|f(\cdot;\theta_f)-S(\cdot)\|_{\infty}$ is small enough, we can guarantee that $\inf_{\theta_f}\sup_I\|g(\cdot; f(I;\theta_f)) - g(\cdot;\hat{S}(I))\|_{\infty} \leq \epsilon$. Indeed, by~\cite{hanin2017approximating}, if $\mathscr{f}$ is a class of large enough ReLU neural networks, we can ensure that $\inf_{\theta_f} \|f(\cdot;\theta_f)-S(\cdot)\|_{\infty}$ is as small as we wish. Hence, by the triangle inequality, we have: $\inf_{\theta_f}\sup_I\| g(\cdot;f(I;\theta_f))-y_I\|_{\infty} \leq 3\epsilon$.
}

\subsection{Parameter Complexity of Meta-Networks}\label{sec:main2}


{As  discussed in Sec.~\ref{sec:main1}, there exists a selection function $S:\mathcal{I} \to \Theta_{\mathscr{g}}$ that takes $I$ and returns parameters of $g$, such that, $g(\cdot;S(I))$ well approximate $y_I$.} In common practical scenarios, the typical assumption regarding the selection function $S(I)$ is that it takes the form $W \cdot h$, for some continuous function $h:\mathcal{I} \to \mathbb{R}^w$ for some relatively small $w>0$ and $W$ is a linear mapping~\cite{pmlr-v95-ukai18a,lorraine2018stochastic,chang2020principled,Littwin_2019_ICCV}. In this section, we show that for functions $y$ with a continuous selector $S$ of this type, the complexity of the function $f$ can be reduced from $\mathcal{O}(\epsilon^{-m/r})$ to $\mathcal{O}(\epsilon^{-m_2/r} + \epsilon^{-m_1/r})$.

\begin{restatable}{theorem}{five}\label{thm:5}
Let $\sigma$ be a in Thm.~\ref{thm:2}. 
Let $\mathscr{g}$ be a class of neural networks with $\sigma$ activations. Let $y\in \mathbb{Y} := \mathcal{W}_{r,m}$ be a target function. Assume that there is a continuous selector $S \in \mathcal{P}_{r,w,c}$ for the class $\{y_I\}_{I \in \mathcal{I}}$ within $\mathscr{g}$. Then, there is a hypernetwork $h(x,I) = g(x;f(I;\theta_f))$ that achieves error $\leq \epsilon$ in approximating $y$, such that: $N_{\mathscr{f}} = \mathcal{O}(w^{1+m_2/r} \cdot \epsilon^{-m_2/r} + w \cdot N_{\mathscr{g}}) = \mathcal{O}(\epsilon^{-m_2/r} + \epsilon^{-m_1/r})$.
\end{restatable}

We note that the number of trainable parameters in a hypernetwork is measured by $N_{\mathscr{f}}$. By Thm.~\ref{thm:generallower}, the number of trainable parameters in a neural network is $\Omega(\epsilon^{-(m_1+m_2)/r})$ in order to be able to approximate any function $y \in \mathcal{W}_{r,m}$. Thm.~\ref{thm:3} shows that in the case of the common hypernetwork structure, the number of trainable parameters of the hypernetwork is reduced to $\mathcal{O}(\epsilon^{-m_2/r} + \epsilon^{-m_1/r})$. While for embedding methods, where the total number of parameters combines those of both $q$ and $e$, it is evident that the overall number of trainable parameters is $\Omega(\epsilon^{-(m_1+m_2)/r})$. In particular, when equating the number of trainable parameters of a hypernetwork with the size of an embedding method, the hypernetworks' approximation error is significantly lower. This kind of stronger rates of realizability is typically associated with an enhanced generalization performance~\cite{NIPS2010_3894}.

\section{Experiments}\label{sec:experiments}


\noindent{\bf Validating Assumption~\ref{assmp:apprx}\quad} Informally, Assumption~\ref{assmp:apprx} claims that for any target function $y \in \mathbb{Y}$, and class $\mathscr{f}$ of neural networks with an activation function $\sigma$, there is a unique global approximator $f^* \in \mathscr{f}$, such that, $f^* \in \arg\inf_{f \in \mathscr{f}} \| f - y\|_{\infty}$. To empirically validate the assumption, we take a high complexity target function $y$ and approximate it using two neural network approximators $f_1$ and $f_2$ of the same architecture $\mathscr{f}$. The goal is to show that when $f_1$ and $f_2$ are best approximators of $y$ within $\mathscr{f}$, then, they have similar input-output relations, regardless of approximation error. 

Three input spaces are considered: (i) the CIFAR10 dataset, (ii) the MNIST dataset and (iii) the set $[-1,1]^{28 \times 28}$. The functions $f_1$ and $f_2$ are shallow ReLU MLP neural networks with $100$ hidden neurons and $10$ output neurons. The target function $y$ is a convolutional neural network of the form:
\begin{equation}
y(x) =  \textnormal{fc}_1\circ  \textnormal{ReLU}\circ \textnormal{conv}_2 \circ \textnormal{ReLU}\circ \textnormal{conv}_1(x)
\end{equation}
where $\textnormal{conv}_1$ ($\textnormal{conv}_2$) is a convolutional layer with $1$ or $3$ ($20$) input channels, $20$ ($50$) output channels, kernel size $10$ and stride $2$ and $\textnormal{fc}_1$ with $10$ outputs.

To study the convergence between $f_1$ and $f_2$, we train them independently to minimize the MSE loss to match the output of $y$ on random samples from the input space. The training was done using the SGD method with a learning rate $\mu=0.01$ and momentum $\gamma = 0.5$, for $50$ epochs. We initialized $f_1$ and $f_2$ using different initializations.

In Fig.~\ref{fig:assmp} we observe that the distance between $f_1$ and $f_2$ tends to be significantly smaller than their distances from $y$. Therefore, we conclude that regardless of the approximation error of $y$ within $\mathscr{f}$, any two best approximators $f_1,f_2\in \mathscr{f}$ of $y$ are identical. 


    

\begin{figure*}[t]
    \centering
    \begin{tabular}{ccc}
    \includegraphics[width=.2532\linewidth]{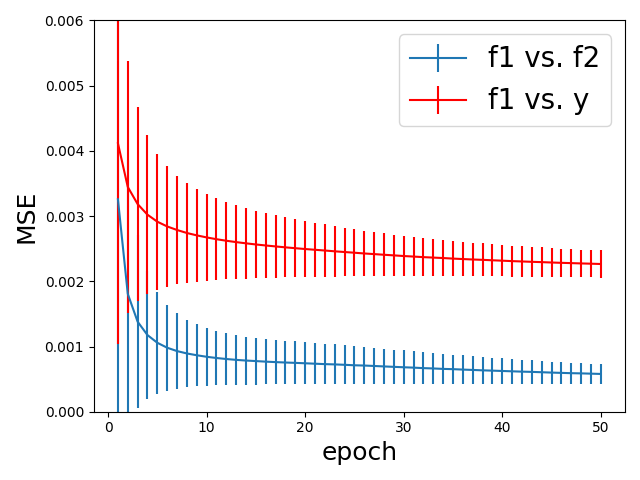}&
    \includegraphics[width=.2532\linewidth]{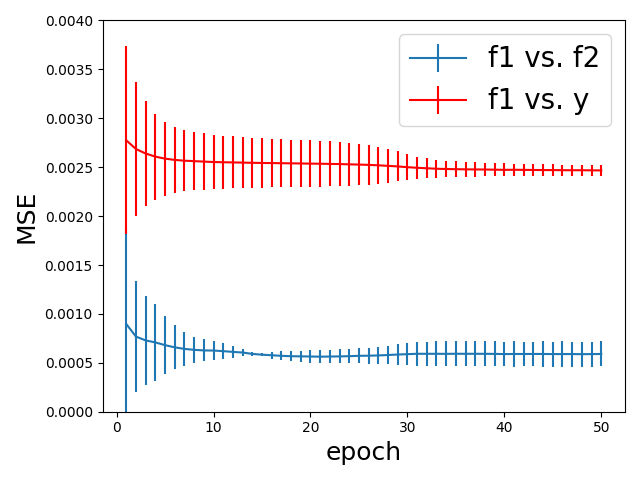}&
    \includegraphics[width=.2532\linewidth]{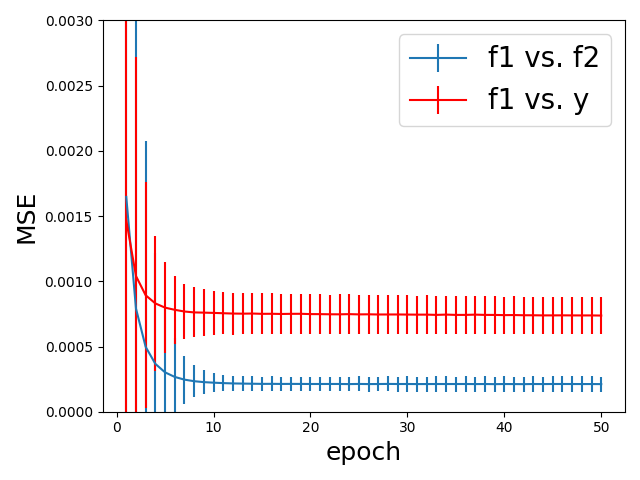}\\
    (a) MNIST & (b) CIFAR10 & (c) The set $[-1,1]^{28\times 28}$
    \end{tabular}
    \caption{{\bf Validating Assumption~\ref{assmp:apprx}.} MSE between $f_1$ and $f_2$ (blue), and between $f_1$ and $y$ (red), when approximating $y$, as a function of epoch.} 
    \label{fig:assmp}
\end{figure*}

\noindent{\bf Synthetic Experiments\quad} We experimented with the following class of target functions. 
The dimensions of $x$ and $I$ are denoted by $d_x$ and $d_I$ (resp.). The target functions is of the form $y(x,I) := \langle x, h(I) \rangle$, $h$ is a three-layers fully-connected sigmoid neural network. 
See the appendix for further details and experiments with two additional classes of target functions.

\noindent{\em Varying the number of layers\quad}  To compare between the two models, we took the primary-networks $g$ and $q$ to be neural networks with two layers of dimensions $d_{\textnormal{in}} \to 10 \to 1$ and ReLU activation within the hidden layer. The input dimension of $g$ is $d_{\textnormal{in}}= d_x = 10^3$ and for $q$ is $d_{\textnormal{in}}=d_x + E = 10^3+10^4$. In addition, $f$ and $e$ are neural networks with $k=2,\dots,9$ layers, each layer of width $100$. {The output dimension of $e$ is $E=10^4$. In this case, the size of $q$ is $N_{\mathscr{q}} = 10^4+10 E + 10$, which is larger than the size of $g$, $N_{\mathscr{g}} = 10^4 + 10$. The sizes of $f$ and $e$ are $N_{\mathscr{f}} = 10^5 + 10^4 \cdot (k-2) + 10^2 \cdot N_{\mathscr{g}}$ and $N_{\mathscr{e}} = 10^5 + 10^4 \cdot (k-2) + 10^6$, which are both of order $10^6$. 
} 

We compared the MSE losses at the test time of the hypernetwork and the embedding method in approximating the target function $y$. The training was done over $30000$ samples $(x,I,y(x,I))$, with $x$ and $I$ taken from a standard normal distribution. The samples are divided into batches of size $200$ and the optimization is done using the SGD method with a learning rate $\mu = 0.01$.

As can be seen in Fig.~\ref{fig:comparison}(a), when the number of layers of $f$ and $e$ are $\geq 3$, the hypernetwork model outperforms the embedding method.
It is also evident that the approximation error of hypernetworks improves, as long as we increase the number of layers of $f$. This is in contrast to the case of the embedding method, the approximation error does not improve when increasing $e$'s number of layers. 
These results are very much in line with the theorems in Sec.~\ref{sec:main2}. As can be seen in Thms.~\ref{thm:2} and~\ref{thm:4}, when fixing the sizes of $g$ and $q$, while letting $f$ and $e$ be as large as we wish we can achieve a much better approximation with the hypernetwork model.  

\noindent{\em Varying the embedding dimension \quad} Next, we investigate the effect of varying the embedding dimension in both models to be $10^2i$, for $i \in [8]$. In this experiment, $d_x=d_I=100$, the primary-networks $g$ and $q$ are set to be ReLU networks with two layers of dimensions $d_{\textnormal{in}} \to 10 \to 1$. The input dimension of $g$ is $d_{\textnormal{in}}=d_x=100$ and for $q$ is $d_{\textnormal{in}}=d_x + 100 i$. The functions $f$ and $e$ are fully connected networks with three layers. The dimensions of $f$ are $10^2 \to 10^2 \to 10^2i \to N_{\mathscr{g}}$ and the dimensions of $e$ are $10^2 \to 10^2 \to 10^2 \to 10^3i$. The overall size of $g$ is $N_{\mathscr{g}} = 1010$ which is smaller than the size of $q$, $N_{\mathscr{q}} = 10^4(i+1)+10$. The size of $f$ is $N_{\mathscr{f}} = 10^4 + 10^4 i + 10^5i$ and the size of $e$ is $N_{\mathscr{e}} = 2\cdot 10^4+10^5 i$ which are both $\approx 10^5 i$.

As can be seen from Fig.~\ref{fig:comparison}(b), the performance of the embedding method does not improve when increasing the embedding dimension. Also, the overall performance is much worse than the performance of hypernetworks with deeper or wider $f$. 
This result verifies the claim in Thm.~\ref{thm:3} that by increasing the embedding dimension the embedding model is unable to achieve the same rate of approximation as the hypernetwork model. 


\begin{figure}
  \begin{minipage}[c]{0.60\textwidth}
  \begin{tabular}{@{}c@{}c@{}}
    \includegraphics[width=.5\linewidth]{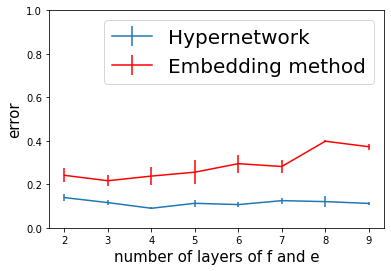}&
    \includegraphics[width=.5\linewidth]{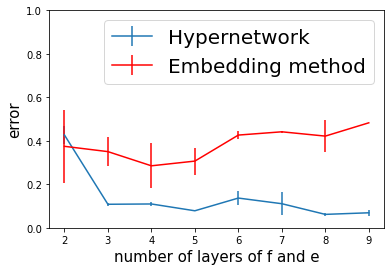}  \\ 
    (a) MNIST &(b) CIFAR10 
    \end{tabular}
 \end{minipage}%
 \hfill
  \begin{minipage}[c]{0.39\textwidth}
    \caption{{\bf Predicting image rotations.} (a-b) The error obtained by hypernetworks and the embedding method with a varying number of layers (x-axis).}
    \label{fig:comparison2}
  \end{minipage}
\end{figure}

\noindent{\bf Experiments on Real-world Datasets\quad}  To validate the prediction in Sec.~\ref{sec:main1}, we experimented with comparing the ability of hypernetworks and embedding methods of similar complexities in approximating the target function. We experimented with the MNIST~\cite{mnist} and CIFAR10 datasets~\cite{cifar} on two self-supervised learning tasks: predicting image rotations, described below and image colorization (Sec.~1.3 in the appendix). For image rotation, the target functions are $y(x,I)$, where $I$ is a sample from the dataset and $x$ is a rotated version of it with a random angle $\alpha$, which is a self-supervised task~\cite{NIPS2019_9697,gidaris2018unsupervised,8953870,10.5555/3327546.3327644}. The function $y$ is the closest value to $\alpha/360$ within $\{\alpha_i = 30i/360 \mid i = 0,\dots,11\}$. The inputs $x$ and $I$ are flattened and their dimensions are $d_x=d_I=h^2 c$, where $h,c$ are the height and number of channels of the images. 

\noindent{\em Varying the number of layers \quad} In this case, the primary-networks $g$ and $q$ are fully connected. 
The input dimension of $g$ is $d_{\textnormal{in}} = d_x$ and of $q$ is $d_{\textnormal{in}} = d_x + N_{\mathscr{g}} = 11h^2 c + 10$. The functions $f$ and $e$ are ReLU neural networks with a varying number of layers $k=2,\dots,9$. Their input dimensions are $d_I$ and each hidden layer in $e$ and $f$ is of dimension $d$. We took $d=50$ for MNIST and $d=100$ for CIFAR10. The output dimensions of $e$ and $f$ are $10h^2c+10$. In this case, the numbers of parameters and output dimensions of $e$ and $f$ are the same, since they share the same architecture. In addition, the number of parameters in $g$ is $N_{\mathscr{g}}=10h^2c+10$, while the number of parameters in $q$ is $N_{\mathscr{q}} = 10(11h^2c+10)+10 \approx 10N_{\mathscr{g}}$. 

We compare the classification errors over the test data. The networks are trained with the negative log loss for $10$ epochs using SGD with a learning rate of $\mu = 0.01$. We did not apply any regularization or normalization on the two models to minimize the influence of hyperparameters on the comparison.

As can be seen in Fig.~\ref{fig:comparison2}, the hypernetwork outperforms the embedding method by a wide margin. In contrast to the embedding method, the hypernetwork's performance improves when increasing its depth. For additional experiments on studying the effect of the embedding dimension, see Sec.~1.2 in the appendix. Finally, since the learning rate is the only hyperparameter in the optimization process, we conducted a sensitivity test, showing that the results are consistent when varying the learning rate (see Sec.~1.4 in the appendix).




\section{Conclusions}

We aim to understand the success of hypernetworks from a theoretical standpoint and compared the complexity of hypernetworks and embedding methods in terms of the number of trainable parameters. To achieve error $\leq \epsilon$ when modeling a function $y(x,I)$ using hypernetworks, the primary-network can be selected to be of a much smaller family of networks than the primary-network of an embedding method. This result manifests the ability of hypernetworks to effectively learn distinct functions for each $y_I$ separately. While our analysis points to the existence of modularity in hypernetworks, it does not mean that this modularity is achievable through SGD optimization. However, our experiments as well as the successful application of this technology in practice, specifically using a large $f$ and a small $g$, indicate that this is indeed the case, and the optimization methods are likely to converge to modular solutions. 



\clearpage

\section*{Broader Impact}
Understanding modular models, in which learning is replaced by meta-learning, can lead to an ease in which models are designed and combined at an abstract level. This way, deep learning technology can be made more accessible. Beyond that, this work falls under the category of basic research and does not seem to have particular societal or ethical implications.

\section*{Acknowledgements and Funding Disclosure}

This project has received funding from the European Research Council (ERC) under the European
Union’s Horizon 2020 research and innovation programme (grant ERC CoG 725974). The contribution of Tomer Galanti is part of Ph.D. thesis research conducted at Tel Aviv University.

\bibliography{hyper}
\bibliographystyle{plain}

\newpage

\section{Additional Experiments}




\subsection{Synthetic Experiments} 

As an additional experiment, we repeated the same experiment (i.e., varying the number of layers of $f$ and $e$ or the embedding dimension) in Sec.~\ref{sec:experiments} with two different classes of target functions (type II and III). The experiments with Type I functions are presented in the main text.

\noindent{\bf Type I\quad}  The target functions is of the form $y(x,I) := \langle x, h(I) \rangle$. Here, $h$ is a three-layers fully-connected neural network of dimensions $d_I \to 300 \to 300 \to 10^3$ and applies sigmoid activations within the two hidden layers and softmax on top of the network. The reason we apply softmax on top of the network is to restrict its output to be bounded. 

\noindent{\bf Type II\quad} The second group of functions consists of randomly initialized fully connected neural networks $y(x,I)$. The neural network has four layers of dimensions $(d_x+d_I)\to 100 \to 50 \to 50 \to 1$ and applies ELU activations. 

\noindent{\bf Type III\quad} The second type of target functions $y(x,I) := h(x \odot I)$ consists of fully-connected neural network applied on top of the element-wise multiplication between $x$ and $I$. The neural network consists of four layers of dimensions $d_I \to 100 \to 100 \to 50 \to 1$ and applies ELU activations. The third type of target functions is of the form $y(x,I) := \langle x, h(I) \rangle$. Here, $h$ is a three-layers fully-connected neural network of dimensions $d_I \to 300 \to 300 \to 1000$ and applies sigmoid activations within the two hidden layers and softmax on top of the network. The reason we apply softmax on top of the network is to restrict its output to be bounded. 

In all of the experiments, the weights of $y$ are set using the He uniform initialization~\cite{10.1109/ICCV.2015.123}.

In Fig.~\ref{fig:supp_comparison1}, we plot the results for varying the number of layers/embedding dimensions of hypernetworks and embedding methods. As can be seen, the performance of hypernetworks improves as a result of increasing the number of layers, despite the embedding method. On the other hand, for both models, increasing the embedding dimension seems ineffective.

\begin{figure*}[t]
    \centering
    \begin{tabular}{ccc}
    \includegraphics[width=.3182\linewidth]{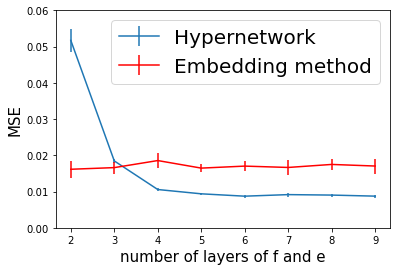}&
    \includegraphics[width=.3182\linewidth]{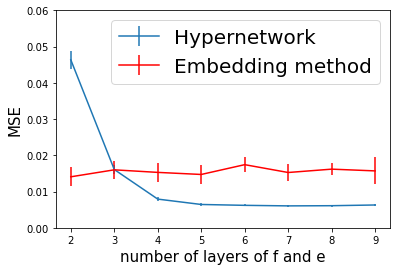}&
    \\
    (a)&(b)
    \\
    \includegraphics[width=.3182\linewidth]{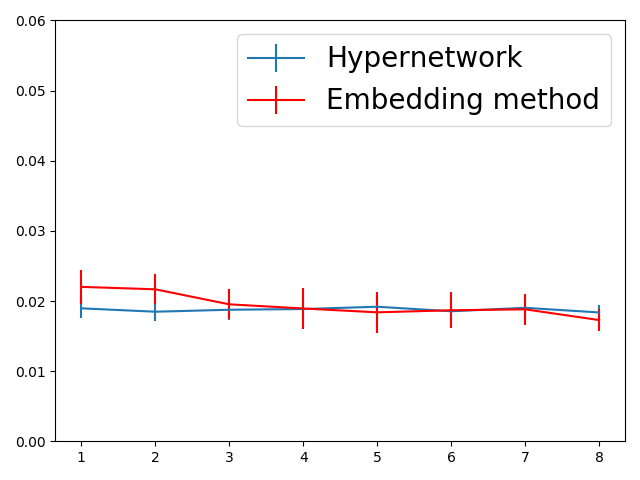}&
    \includegraphics[width=.3182\linewidth]{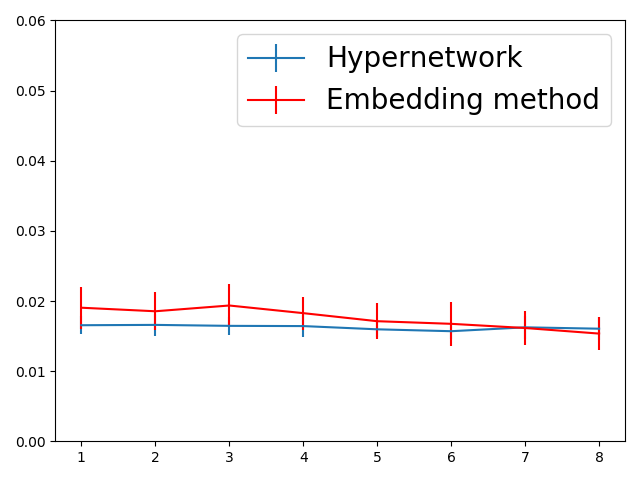}&
    \\
    (d)&(e)
    \\
    \end{tabular}
    \caption{(a-b) The error obtained by hypernetworks and the embedding method with varying number of layers (x-axis). The MSE (y-axis) is computed between the learned function and the target function at test time. The blue curve stands for the performance of the hypernetwork model and the red one for the neural embedding method. (a) Target functions of neural network type, 
    (b) Functions of the form $y(x,I) = h(x \odot I)$, where $h$ is a neural network.
     (d-e) Measuring the performance for the same three target functions when varying the size of the embedding layer to be 100/1000 (depending on the method) times the value on the x-axis. The error bars depict the variance across 100 repetitions of the experiment.}
    \label{fig:supp_comparison1}
\end{figure*}


\subsection{Predicting Image Rotations}

As an additional experiment on predicting image rotations, we studied the effect of the embedding dimension on the performance of the embedding method, we varied the embedding dimension $E_i=10^4i$ for $i\in [8]$. The primary-network $q$ has dimensions $d_{\textnormal{in}} \to 10 \to 12$ with $d_{\textnormal{in}} = d_I + E_i$ and the embedding network $e$ has architecture $d_x \to 100 \to E_i$. We compared the performance to a hypernetwork with $g$ of architecture $d_x \to 10 \to 12$ and $f$ of architecture $d_I \to 100 \to N_{\mathscr{g}}$. We note that $q$ is larger than $g$, the embedding dimension $E_i$ exceeds $N_{\mathscr{g}} = 30840$ for any $i>3$ and therefore, $e$ is of larger size than $f$ for $i>3$.

As can be seen in Fig.~\ref{fig:comparison3}, the hypernetwork outperforms the embedding method by a wide margin and the performance of the embedding method does not improve when increasing its embedding dimension.

\begin{figure*}[t]
    \centering
    \begin{tabular}{c@{~}c@{~}c@{~}c}
    \includegraphics[width=.34832\linewidth]{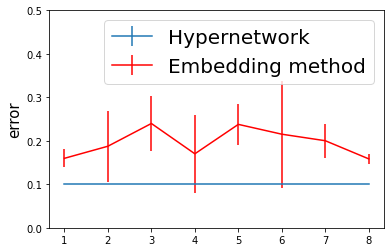}& 
    \includegraphics[width=.34832\linewidth]{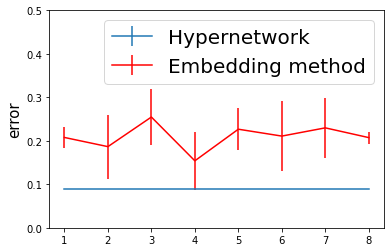} \\
    (a) MNIST & (b) CIFAR10
    \end{tabular}
    \caption{{\bf Predicting image rotations.}  varying the embedding dimension of the embedding method to be $10^4$ times the value of the x-axis, compared to the results of hypernetworks. The error bars depict the variance across 100 repetitions of the experiment. }
    \label{fig:comparison3}
\end{figure*}

\subsection{Image Colorization} 

The second type of target functions are $y(x,I)$, where $I$ is a sample gray-scaled version of an image $\hat{I}$ from the dataset and $x=(i_1,i_2)$ is a tuple of coordinates, specifying a certain pixel in the image $I$. The function $y(x,I)$ returns the RGB values of $\hat{I}$ in the pixel $x = (i_1,i_2)$ (normalized between $[-1,1]$). For this self-supervised task we employ CIFAR10 dataset, since the MNIST has grayscale images.  

For the purpose of comparison, we considered the following setting. The inputs of the networks are $x' = (i_1,i_2) \| (i^k_1 + i_2,i^k_2 + i_1,i^k_1 - i_2,i^k_2 - i_1)^{9}_{k=0}$ and a flattened version of the gray-scaled image $I$ of dimensions $d_{x'}=42$ and $d_I=1024$. The functions $f$ and $e$ are fully connected neural networks of the same architecture with a varying number of layers $k=2,\dots,7$. Their input dimension is $d_I$, each hidden layer is of dimension $100$ and their output dimensions are $450$. We took primary networks $g$ and $q$ to be fully connected neural networks with two layers $d_{\textnormal{in}} \to 10 \to 3$ and ELU activations within their hidden layers. For the hypernetwork case, we have: $d_{\textnormal{in}} = 42$ and for the embedding method $d_{\textnormal{in}}=42+450 = 492$, since the input of $q$ is a concatenation of $x'$ (of dimension $42$) and $e(I)$ which is of dimension $450$. 

The overall number of trainable parameters in $e$ and $f$ is the same, as they share the same architecture. The number of trainable parameters in $q$ is $492 \cdot 10 + 10 \cdot 3 = 4950$ and in $g$ is $42 \cdot 10 + 10 \cdot 3 = 450$. Therefore, the embedding method is provided with a larger number of trainable parameters as $q$ is $10$ times larger than $g$. The comparison is depicted in Fig.~\ref{fig:supp_comparison2}. As can be seen, the results of hypernetworks outperform the embedding method by a large margin, and the results improve when increasing the number of layers. 

\begin{figure*}[ht]
    \centering
    \begin{tabular}{ccc}
    \includegraphics[width=.3132\linewidth]{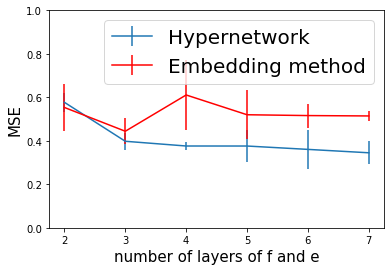}\\
    \end{tabular}
    \caption{{\bf Colorization.} The error obtained by hypernetworks and the embedding method with varying number of layers (x-axis). The error rate (y-axis) is computed between the learned function and the target function at test time. The blue curve stands for the performance of the hypernetwork model and the red one for the neural embedding method.}
    \label{fig:supp_comparison2}
\end{figure*}

\subsection{Sensitivity Experiment}

\begin{figure*}[ht]
    \centering
    \begin{tabular}{cc}
    \includegraphics[width=.3832\linewidth]{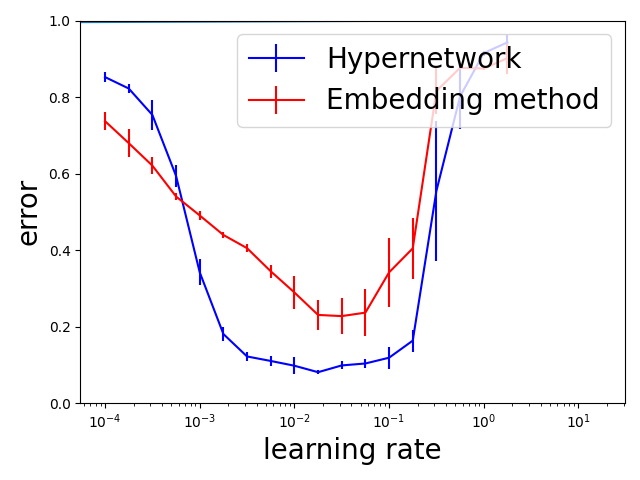} & 
    \includegraphics[width=.3832\linewidth]{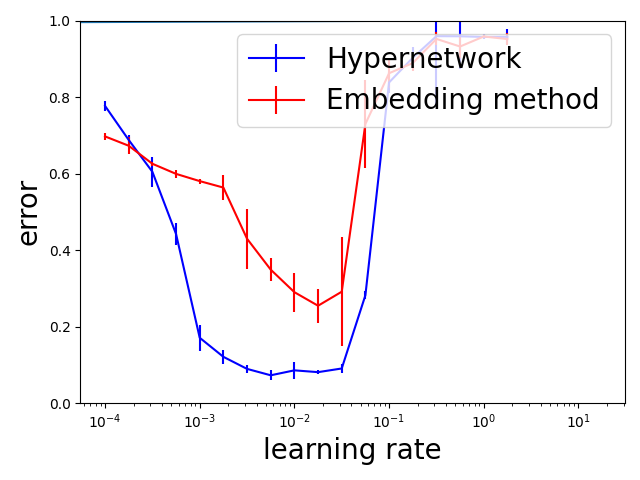} \\
    (a) & (b)
    \end{tabular}
    \caption{{\bf Comparing the performance of a hypernetwork and the embedding method when varying the learning rate.} The x-axis stands for the value of the learning rate and the y-axis stands for the averaged accuracy rate at test time. (a) Results on MNIST and (b) Results on CIFAR10. }
    \label{fig:sensitivity}
\end{figure*}

In the rotations prediction experiment in Sec.~\ref{sec:experiments}, we did not apply any regularization or normalization on the two models to minimize the number of hyperparameters. Therefore, the only hyperparameter we used during the experiment is the learning rate. We conducted a hyperparameter sensitivity test for the learning rate. We compared the two models in the configuration of Sec.~\ref{sec:experiments} when fixing the depths of $f$ and $e$ to be $4$ and varying the learning rate. As can be seen in Fig.~\ref{fig:sensitivity}, the hypernetwork outperforms the baseline for every learning rate in which the networks provide non-trivial error rates.

\subsection{Validating Assumption~2}

To empirically justify Assumption~2, we trained shallow neural networks on MNIST and Fashion MNIST classification with a varying number of hidden neurons. The optimization was done using the MSE loss, where the labels are cast into one-hot encoding. The network is trained using Adadelta with a learning rate of $\mu=1.0$ and batch size $64$ for $2$ epochs. As can be seen in Fig.~\ref{fig:assmp2}, the MSE loss strictly decreases when increasing the number of hidden neurons. This is true for a variety of activation functions.

\begin{figure*}[t]
    \centering
    \begin{tabular}{ccc}
    \includegraphics[width=.3532\linewidth]{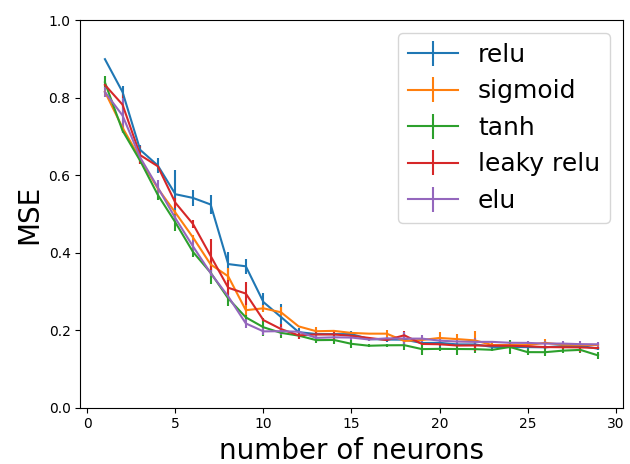}&
    \includegraphics[width=.3532\linewidth]{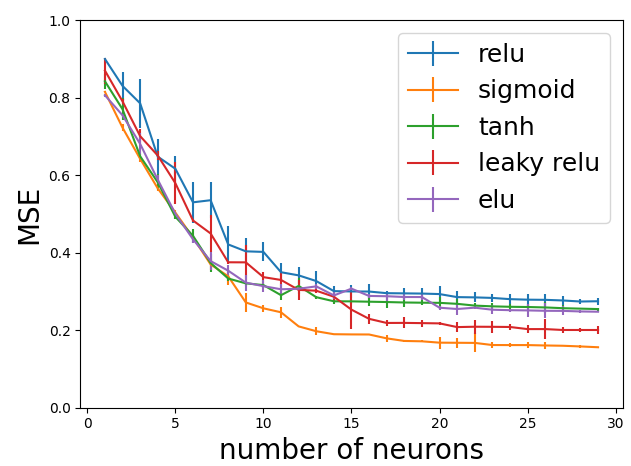}\\
    (a) MNIST & (b) Fashion MNIST
    \end{tabular}
    \caption{{\bf Validating Assumption~2.} The MSE loss at test time strictly decreases when increasing the number of hidden neurons.} 
    \label{fig:assmp2}
\end{figure*}

\newpage

\section{Preliminaries}

\subsection{Identifiability}\label{sec:identify}

Neural network identifiability is the property in which the input-output map realized by a feed-forward neural network with respect to a given activation function uniquely specifies the network architecture, weights, and biases of the neural network up to neural network isomorphisms (i.e., re-ordering the neurons in the hidden layers). Several publications investigate this property. For instance, \cite{Albertini93uniquenessof,Sussmann1992UniquenessOT} show that shallow neural networks are identifiable. The main result of~\cite{fefferman} considers feed-forward neural networks with the $\tanh$ activation functions are shows that these are identifiable when the networks satisfy certain ``genericity assumptions``. In~\cite{nn-id-2019} it is shown that for a wide class of activation functions, one can find an arbitrarily close function that induces identifiability (see Lem.~\ref{lem:vlacic}). Throughout the proofs of our Thm.~\ref{thm:generallower}, we make use of this last result in order to construct a robust approximator for the target functions of interest.  

We recall the terminology of identifiability from~\cite{fefferman,nn-id-2019}. 

\begin{definition}[Identifiability]
A class $\mathscr{f} = \{f(\cdot;\theta_f):A \to B \;\vert\; \theta_f \in \Theta_{\mathscr{f}}\}$ is identifiable up to (invariance) continuous functions $\Pi = \{\pi:\Theta_{\mathscr{f}} \to \Theta_{\mathscr{f}}\}$, if
\begin{equation}
f(\cdot;\theta_f) \equiv_A f(\cdot;\theta'_f) \iff \exists \pi \in \Pi  \textnormal{ s.t } \theta'_f = \pi(\theta_f)
\end{equation}
where the equivalence $\equiv_A$ is equality for all $x \in A$.
\end{definition}

A special case of identifiability is identifiability up to isomorphisms. Informally, we say that two neural networks are isomorphic if they share the same architecture and are equivalent up to permuting the neurons in each layer (excluding the input and output layers). 

\begin{definition}[Isomorphism]\label{def:isom}
Let $\mathscr{f}$ be a class of neural networks. Two neural networks $f(x;[\bfem{W},\bfem{b}])$ and $f(x;[\bfem{V},\bfem{d}])$ of the same class $\mathscr{f}$ are isomorphic if there are permutations $\{\gamma_i:[h_i] \to [h_i]\}^{k+1}_{i=1}$, such that,
\begin{enumerate}
\item $\gamma_1$ and $\gamma_{k+1}$ are the identity permutations.
\item For all $i \in [k]$, $j \in [h_{i+1}]$ and $l \in [h_{i}]$, we have: $V^i_{j,l} = W^i_{\gamma_{i+1}(j),\gamma_{i}(l)} \textnormal{ and } d^i_j = b^i_{\gamma_{i+1}(j)}$.
\end{enumerate} 
An isomorphism $\pi$ is specified by permutation functions $\gamma_1,\dots,\gamma_{k+1}$ that satisfy conditions (1) and (2). For a given neural network $f(x;[\bfem{W},\bfem{b}])$ and isomorphism $\pi$, we denote by $\pi \circ [\bfem{W},\bfem{b}]$ the parameters of a neural network produced by the isomorphism $\pi$.
\end{definition}

As noted by~\cite{fefferman,nn-id-2019}, for a given class of neural networks, $\mathscr{f}$, there are several ways to construct pairs of non-isomorphic neural networks that are equivalent as functions. 

In the first approach, suppose that we have a neural network with depth $k \geq 2$,
and there exist indices $i, j_1, j_2$ with $1 \leq i \leq k-1$ and $1 \leq j_1 < j_2 \leq h_{i+1}$, such that, $b^i_{j_1} = b^i_{j_2}$ and $W^{i}_{j_1,t} = W^{i}_{j_2,t}$ for all $t \in [h_{i}]$. Then, if we construct a second neural network that shares the same weights and biases, except replacing $W^{i+1}_{1,j_1}$ and $W^{i+1}_{1,j_2}$ with a pair $\tilde{W}^{i+1}_{1,j_1}$ and $\tilde{W}^{i+1}_{1,j_2}$, such that, $\tilde{W}^{i+1}_{1,j_1}+\tilde{W}^{i+1}_{1,j_2} = W^{i+1}_{1,j_1}+W^{i+1}_{1,j_2}$. Then, the two neural networks are equivalent, regardless of the activation function. The $j_1$ and $j_2$ neurons in the $i$'th layer are called clones and are defined formally in the following manner.

\begin{definition}[No-clones condition] Let class of neural networks $\mathscr{f}$. Let $f(x;[\bfem{W},\bfem{b}]) \in \mathscr{f}$ be a neural network. We say that $f$ has clone neurons if there are: $i \in [k]$, $j_1\neq j_2 \in [h_{i+1}]$, such that:
\begin{equation}
(b^i_{j_1},W^i_{j_1,1},\dots,W^i_{j_1,h_{i}}) = (b^i_{j_2},W^i_{j_2,1},\dots,W^i_{j_2,h_{i}})
\end{equation}
If $f$ does not have a clone, we say that $f$ satisfies the no-clones condition.  
\end{definition}

A different setting in which uniqueness up to isomorphism is broken, results when taking a neural network that has a ``zero'' neuron. Suppose that we have a neural network with depth $k \geq 2$,
and there exist indices $i, j$ with $1 \leq i \leq k-1$ and $1 \leq j \leq h_{i+1}$, such that, $W^{i}_{j,t} = 0$ for all $t \in [h_{i}]$ or $W^{i+1}_{t,j} = 0$ for all $t \in [h_{i+2}]$. In the first case, one can replace any $W^{i+1}_{1,j}$ with any number $\tilde{W}^{i+1}_{1,j}$ if $\sigma(b_{i,j}) = 0$ to get a non-isomorphic equivalent neural network. In the other case, one can replace $W^{i}_{j,1}$ with any number $\tilde{W}^{i+1}_{j,1}$ to get non-isomorphic equivalent neural network. 

\begin{definition}[Minimality] Let $f(x;[\bfem{W},\bfem{b}])$ be a neural network. We say that $f$ is minimal, if for all $i \in [k]$, each matrix $W^i$ has no identically zero row or an identically zero column.
\end{definition}

A normal neural network satisfies both minimality and the no-clones condition.
\begin{definition}[Normal neural network]\label{def:normal}
Let $f(x;[\bfem{W},\bfem{b}])$ be a neural network. We say that $f$ is normal, if it has no-clones and is minimal. The set of normal neural networks within $\mathscr{f}$ is denoted by $\mathscr{f}_n$.
\end{definition}

An interesting question regarding identifiability is whether a given activation $\sigma:\mathbb{R} \to \mathbb{R}$ function implies the identifiability property of any class of normal neural networks $\mathscr{f}_n$ with the given activation function are equivalent up to isomorphisms. An activation function of this kind will be called {\em identifiability inducing}.
It has been shown by~\cite{fefferman} that the $\tanh$ is identifiability inducing up to additional restrictions on the weights. In~\cite{Sussmann1992UniquenessOT} and in~\cite{Albertini93uniquenessof} they show that shallow neural networks are identifiable.

\begin{definition}[Identifiability inducing activation]\label{def:induc} Let $\sigma: \mathbb{R} \to \mathbb{R}$ be an activation function. We say that $\sigma$ is identifiability inducing if for any class of neural networks $\mathscr{f}$ with $\sigma$ activations, we have: $f(\cdot;\theta_1) = f(\cdot;\theta_2) \in \mathscr{f}_n$ if and only if they are isomorphic. 
\end{definition}

The following theorem by~\cite{nn-id-2019} shows that any piece-wise $C^1(\mathbb{R})$ activation function $\sigma$ with $\sigma' \in BV(\mathbb{R})$ can be approximated by an identifiability inducing activation function $\rho$. 

\begin{lemma}[\cite{nn-id-2019}]\label{lem:vlacic}  Let $\sigma:\mathbb{R} \to \mathbb{R}$ be a piece-wise $C^1(\mathbb{R})$ with $\sigma' \in BV (\mathbb{R})$ and let $\epsilon > 0$. Then, there exists a meromorphic function $\rho : D \to \mathbb{C}$, $\mathbb{R} \subset D$, $\rho(\mathbb{R}) \subset \mathbb{R}$, such that, $\| \sigma - \rho\|_{\infty} < \epsilon$ and $\rho$ is identifiability inducing.
\end{lemma}


\subsection{Multi-valued Functions}

Throughout the proofs, we will make use of the notion of multi-valued functions and their continuity. A multi-valued function is a mapping $F:A \to \mathcal{P}(B)$ from a set $A$ to the power set $\mathcal{P}(B)$ of some set $B$. To define the continuity of $F$, we recall the Hausdorff distance~\cite{GlossarWiki:Hausdorff:1914,RockWets98} between sets. Let $d_B$ be a distance function over a set $B$, the Hausdorff distance between two subsets $E_1,E_2$ of $B$ is defined as follows:
\begin{equation}
\begin{aligned}
d_{\mathcal{H}}(E_1,E_2) := \max \Big\{\sup_{b_1 \in E_1} \inf_{b_2 \in E_2} d_B(b_1,b_2) , \sup_{b_2 \in E_2} \inf_{b_1 \in E_1} d_B(b_1,b_2) \Big\}
\end{aligned}
\end{equation}
In general, the Hausdorff distance serves as an extended pseudo-metric, i.e., satisfies $d_{\mathcal{H}}(E,E)=0$ for all $E$, is symmetric and satisfies the triangle inequality, however, it can attain infinite values and there might be $E_1\neq E_2$, such that, $d_{\mathcal{H}}(E_1,E_2)=0$. When considering the space $\mathcal{C}(B)$ of non-empty compact subsets of $B$, the Hausdorff distance becomes a metric. 

\begin{definition}[Continuous multi-valued functions] Let metric spaces $(A,d_A)$ and $(B,d_B)$ and multi-valued function $F:A \to \mathcal{C}(B)$. Then, we define:
\begin{enumerate}[leftmargin=*]
\item {\bf Convergence:} we denote $E = \lim_{a \to a_0} F(a)$, if $E$ is a compact subset of $B$ and it satisfies:
\begin{equation}
\begin{aligned}
\lim_{a \to a_0} d_{\mathcal{H}}(F(a),E) = 0
\end{aligned}
\end{equation}
\item {\bf Continuity:} we say that $F$ is continuous in $a_0$, if  $\lim_{a \to a_0} F(a) = F(a_0)$. 
\end{enumerate}
\end{definition}

\subsection{Lemmas}

In this section, we provide several lemmas that will be useful throughout the proofs of the main results.

Let $[\bfem{W}^1,\bfem{b}^1]$ and $[\bfem{W}^2,\bfem{b}^2]$ be two parameterizations. We denote by $[\bfem{W}^1,\bfem{b}^1] - [\bfem{W}^2,\bfem{b}^2] = [\bfem{W}^1-\bfem{W}^2,\bfem{b}^1-\bfem{b}^2]$ the element-wise subtraction between the two parameterizations. In addition, we define the $L_2$-norm of $[\bfem{W},\bfem{b}]$ to be:
\begin{equation}
\big\|[\bfem{W},\bfem{b}]\big\|_2 := \|\textnormal{vec}([\bfem{W},\bfem{b}])\|_2 := \sqrt{\sum^{k}_{i=1}(\| W^i\|^2_2 + \|b^i\|^2_2)} 
\end{equation}

\begin{lemma}\label{lem:isoBasics}
Let $f(x;[\bfem{W}^1,\bfem{b}^1])$ and $f(x;[\bfem{W}^2,\bfem{b}^2])$ be two neural networks. Then, for a given isomorphism $\pi$, we have:
\begin{equation}
\pi \circ [\bfem{W}^1,\bfem{b}^1] - \pi \circ [\bfem{W}^2,\bfem{b}^2] = \pi\circ [\bfem{W}^1-\bfem{W}^2,\bfem{b}^1-\bfem{b}^2]
\end{equation}
and
\begin{equation}
\big\| \pi \circ [\bfem{W},\bfem{b}] \big\|_2 = \big\|[\bfem{W},\bfem{b}]\big\|_2
\end{equation}
\end{lemma}

\proof{Follows immediately from the definition of isomorphisms.}

\begin{lemma}\label{lem:boundnet}
Let $\sigma:\mathbb{R} \to \mathbb{R}$ be a $L$-Lipschitz continuous activation function, such that, $\sigma(0)=0$. Let $f(\cdot;[\bfem{W},0]):\mathbb{R}^m \to \mathbb{R}$ be a neural network with zero biases. Then, for any $x \in \mathbb{R}^{m}$, we have:
\begin{equation}
\|f(x;[\bfem{W},0])\|_{1} \leq L^{k-1}\cdot \|x\|_1 \prod^{k}_{i=1}  \|W^i\|_{1}
\end{equation}
\end{lemma}

\begin{proof}
Let $z = W^{k-1} \cdot \sigma(\dots \sigma(W^1 x))$. We have:
\begin{equation}
\begin{aligned}
\|f(x;[\bfem{W},0])\|_{1} &\leq \|W^{k} \cdot \sigma(z)\|_{1} \\
&\leq \|W^{k} \cdot \sigma(z)\|_{1} \\
&= \|W^{k}\|_1 \cdot \| \sigma(z)- \sigma(0)\|_{1} \\
&\leq \|W^{k}\|_1 \cdot L \cdot \| z\|_{1} \\
\end{aligned}
\end{equation}
and by induction we have the desired.
\end{proof}

\begin{lemma}\label{lem:boundlip}
Let $\sigma:\mathbb{R} \to \mathbb{R}$ be a $L$-Lipschitz continuous activation function, such that, $\sigma(0)=0$. Let $f(\cdot;[\bfem{W},\bfem{b}])$ be a neural network. Then, the Lipschitzness of $f(\cdot;[\bfem{W},\bfem{b}])$ is given by:
\begin{equation}
\textnormal{Lip}(f(\cdot;[\bfem{W},\bfem{b}])) \leq L^{k-1} \cdot \prod^{k}_{i=1}  \|W^i\|_{1}
\end{equation}
\end{lemma}

\begin{proof}
Let $z_i = W^{k-1} \cdot \sigma(\dots \sigma(W^1 x_i + b^1))$ for some $x_1$ and $x_2$. We have:
\begin{equation}
\begin{aligned}
\|f(x_1;[\bfem{W},\bfem{b}]) - f(x_2;[\bfem{W},\bfem{b}])\|_{1} &\leq \|W^{k} \cdot \sigma(z_1) - W^{k} \cdot \sigma(z_2)\|_{1} \\
&\leq \|W^{k} \cdot (\sigma(z_1 + b^{k-1})-\sigma(z_2 + b^{k-1}))\|_{1} \\
&= \|W^{k}\|_1 \cdot \| \sigma(z_1 + b^{k-1})- \sigma(z_2 + b^{k-1})\|_{1} \\
&\leq \|W^{k}\|_1 \cdot L \cdot \| z_1 - z_2 \|_{1} \\
\end{aligned}
\end{equation}
and by induction we have the desired.
\end{proof}

Throughout the appendix, a function $y \in \mathbb{Y}$ is called {\em normal} with respect to $\mathscr{f}$, if it has a best approximator $f \in \mathscr{f}$, such that, $f \in \mathscr{f}_n$. 

\begin{lemma}\label{lem:normal}
Let $\mathscr{f}$ be a class of neural networks. Let $y$ be a target function. Assume that $y$ has a best approximator $f \in \mathscr{f}$. If $y\notin \mathscr{f}$, then, $f\in \mathscr{f}_n$.
\end{lemma}

\begin{proof}
Let $f(\cdot;[\bfem{W},\bfem{b}]) \in \mathscr{f}$ be the best approximator of $y$. Assume it is not normal. Then, $f(\cdot;[\bfem{W},\bfem{b}])$ has at least one zero neuron or at least one pair of clone neurons. Assume it has a zero neuron. Hence, by removing the specified neuron, we achieve a neural network of architecture smaller than $\mathscr{f}$ that achieves the same approximation error as $\mathscr{f}$ does. This is in contradiction to Assumption~2. For clone neurons, we can simply merge them into one neuron and obtain a smaller architecture that achieves the same approximation error, again, in contradiction to Assumption~2.
\end{proof}

\begin{lemma}\label{lem:cont1}
Let $\mathscr{f}$ be a class of functions with a continuous activation function $\sigma$. Let $\mathbb{Y}$ be a class of target functions. Then, the function $\|f(\cdot;\theta) - y\|_{\infty}$ is continuous with respect to both $\theta$ and $y$ (simultaneously).
\end{lemma}

\begin{proof}
Let sequences $\theta_n \to \theta_0$ and $y_n \to y_0$. By the reversed triangle inequality, we have:
\begin{equation}
\begin{aligned}
&\Big\vert \|f(\cdot;\theta_n) - y_n\|_{\infty} - \| f(\cdot;\theta_0) - y_0 \|_{\infty} \Big\vert \leq \|f(\cdot;\theta_n) - f(\cdot;\theta_0)\|_{\infty} + \|y_n - y_0\|_{\infty}
\end{aligned}
\end{equation}
Since $\theta_n \to \theta_0$ and $f$ is continuous with respect to $\theta$, we have: $\|f(\cdot;\theta_n) - f(\cdot;\theta_0)\|_{\infty} \to 0$. Hence, the upper bound tends to $0$.
\end{proof}

\begin{lemma}\label{lem:cont}
Let $\mathscr{f}$ be a class of functions with a continuous activation function $\sigma$. Let $\mathbb{Y}$ be a closed class of target functions. Then, the function $F(y) := \min_{\theta \in \Theta_{\mathscr{f}}}\|f(\cdot;\theta) - y\|_{\infty}$ is continuous with respect to $y$.
\end{lemma}

\begin{proof}
Let $\{y_n\}^{\infty}_{n=1} \subset \mathbb{Y}$ be a sequence that converges to some $y_0 \in \mathbb{Y}$. Assume by contradiction that:
\begin{equation}
\lim_{n\to \infty} F(y_n) \neq F(y_0)
\end{equation}
Then, there is a sub-sequence $y_{n_k}$ of $y_n$, such that, $\forall k \in \mathbb{N}: F(y_{n_k}) - F(y_0) > \Delta$ or $\forall k \in \mathbb{N}: F(y_0) - F(y_{n_k}) > \Delta$ for some $\Delta > 0$. Let $\theta_0$ be the minimizer of $\|f(\cdot;\theta) - y_0\|_{\infty}$. With no loss of generality, we can assume the first option. We notice that:
\begin{equation}
\begin{aligned}
F(y_{n_k}) \leq \|f(\cdot;\theta_0) - y_{n_k}\|_{\infty} 
\leq \|f(\cdot;\theta_0) - y_0\|_{\infty} + \| y_{n_k} - y_0 \|_{\infty} \leq F(y_0) + \delta_k
\end{aligned}
\end{equation}
where $\delta_k := \| y_{n_k} - y_0 \|_{\infty}$ tends to $0$. This contradicts the assumption that $F(y_{n_k}) > F(y_0) + \Delta$.
\end{proof}

Throughout the appendix, we will make use of the following notation. Let $y \in \mathbb{Y}$ be a function and $\mathscr{f}$ a class of functions, we define:
\begin{equation}
M[y;\mathscr{f}] := \arg\min\limits_{\theta \in \Theta_{\mathscr{f}}} \|f(\cdot;\theta) - y \|_{\infty}
\end{equation}

\begin{lemma}\label{lem:ball}
Let $\mathscr{f}$ be a class of neural networks with a continuous activation function $\sigma$. Let $\mathbb{Y}$ be a class of target functions. Denote by $f_y$ the unique approximator of $y$ within $\mathscr{f}$. Then, $f_y$ is continuous with respect to $y$.
\end{lemma}

\begin{proof}
Let $y_0 \in \mathbb{Y}$ be some function. Assume by contradiction that there is a sequence $y_n \to y_0$, such that, $g_n := f_{y_n} \not\to f_{y_0}$. Then, $g_n$ has a sub-sequence that has no cluster points or it has a cluster point $h \neq f_{y_0}$.

{\bf Case 1:} Let $g_{n_k}$ be a sub-sequence of $g_n$ that has no cluster points. By Assumption~1, there is a sequence $\theta_{n_k} \in \cup^{\infty}_{k=1} M[y_{n_k};\mathscr{f}]$ that is bounded in $\mathbb{B} = \{\theta \mid \|\theta\|_2\leq B\}$. By the Bolzano-Weierstrass' theorem, it includes a convergent sub-sequence $\theta_{n_{k_i}} \to \theta_0$. Therefore, we have:
\begin{equation}
\|f(\cdot;\theta_{n_{k_i}}) - f(\cdot;\theta_0) \|_{\infty} \to 0
\end{equation}
Hence, $g_{n_k}$ has a cluster point $f(\cdot;\theta_0)$ in contradiction.

{\bf Case 2:} Let sub-sequence $f_{y_{n_k}}$ that converge to a function $h \neq f_{y_0}$. We have:
\begin{equation}
\begin{aligned}
\| h - y_0 \|_{\infty} \leq \|f_{y_{n_k}} - h \|_{\infty} + \|f_{y_{n_k}} - y_{n_k} \|_{\infty} + \|y_{n_k} - y_0 \|_{\infty}
\end{aligned}
\end{equation}
By Lem.~\ref{lem:cont},
\begin{equation}
\|f_{y_{n_k}} - y_{n_k} \|_{\infty} \to \|f_{y_0} - y_0\|_{\infty}
\end{equation}
and also $y_{n_k} \to y_0$, $f_{y_{n_k}} \to h$. Therefore, we have:
\begin{equation}
\begin{aligned}
\|h - y_0 \|_{\infty} \leq \|f_{y_0} - y_0\|_{\infty}
\end{aligned}
\end{equation}
Hence, since $f_{y_0}$ is the unique minimizer, we conclude that $h = f_{y_0}$ in contradiction. 

Therefore, we conclude that $f_{y_n}$ converges and by the analysis in Case 2 it converges to $f_{y_0}$.
\end{proof}

\section{Proofs of the Main Results}

\subsection{Proving Assumption~2 for Shallow Networks}

\begin{lemma}
Let $\mathbb{Y} = C([-1,1]^m)$ be the class of continuous functions $y:[-1,1]^m \to \mathbb{R}$. Let $\mathscr{f}$ be a class of 2-layered neural networks of width $d$ with $\sigma$ activations, where $\sigma$ is either $\tanh$ or sigmoid. Let $y \in \mathbb{Y}$ be some function to be approximated. Let $\mathscr{f}'$ be a class of neural networks that is resulted by adding a neuron to the hidden layer of $\mathscr{f}$. If $y \notin \mathscr{f}$ then, $\inf_{\theta \in \Theta_{\mathscr{f}}} \| f(\cdot;\theta) - y\|^2_{2} > \inf_{\theta \in \Theta_{\mathscr{f}'}} \| f(\cdot;\theta) - y\|^2_{2}$. The same holds for $\sigma = ReLU$ when $m=1$.
\end{lemma}

\begin{proof} We divide the proof into two parts. In the first part we prove the claim for neural networks with ReLU activations and in the second part, for the $\tanh$ and sigmoid activations.

\noindent{\bf ReLU activations\quad} Let $y \in \mathbb{Y}$ be a non-piecewise linear function. Let $f \in \mathscr{f}$ be the best approximator of $y$. Since $f$ is a 2-layered neural network, it takes the form:
\begin{equation}
f(x) = \sum^{d}_{i=1} \beta_i \cdot \sigma(\alpha_i x + \gamma_i)
\end{equation}
By~\cite{arora2018understanding}, we note that $f$ is a piece-wise linear function with $k$ pieces. We denote the end-points of those pieces by: $-1=c_0,\dots,c_k=1$. Since $y$ is a non-piecewise linear function, there exists a pair $c_i,c_{i+1}$, where $y$ is non-linear on $[c_i,c_{i+1}]$. With no loss of generality, we assume that $y$ is non-linear on the first segment. We note that $f$ equals some linear function $ax+b$ over the segment $[-1,c_1]$. We would like to prove that one is able to add a new neuron $n(x) = \beta_{d+1} \cdot \sigma(\gamma_{d+1} - x)$ to $f$, for some $-1 < \gamma_{d+1} < c_1$, such that, $f(x)+n(x)$ strictly improves the approximation of $f$. First, we notice that this neuron is non-zero only when $x < \gamma_{d+1}$. Therefore, for any $\beta_{d+1} \in \mathbb{R}$ and $-1 < \gamma_{d+1} < c_1$, $f(x) + n(x) = f(x) \in [c_1,1]$. In particular, the approximation error of $f(x)+n(x)$ over $[c_1,1]$ is the same as $f$'s. For simplicity, we denote $\gamma:=\gamma_{d+1}$ and $\beta := \beta_{d+1}$. Assume by contradiction that there are no such $\gamma$ and $\beta$. Therefore, for each $\gamma \in [-1,c_1]$, $ax+b$ is the best linear approximator of $y(x)$ in the segment $[-1,\gamma]$. Hence, for each $\gamma \in [-1,c_1]$, $\beta = 0$ is the minimizer of $\int^{\gamma}_{-1} (y(x) - (\beta (\gamma - x) + ax+b))^2~dx$. In particular, we have:
\begin{equation}\label{eq:diffzero}
\frac{\int^{\gamma}_{-1} (y(x) - (\beta (\gamma - x) + ax+b))^2~dx }{\partial \beta} \Big|_{\beta=0} = 0
\end{equation}
By differentiation under the integral sign:
\begin{equation}
\begin{aligned}
Q(\beta,\gamma)=&\frac{\int^{\gamma}_{-1} (y(x) - (\beta (\gamma - x) + ax+b))^2~dx }{\partial \beta} \\
&\int^{\gamma}_{-1}  \frac{(y(x) - (\beta (\gamma - x) + ax+b))^2  }{\partial \beta}~dx \\
=& \int^{\gamma}_{-1} 2(y(x) - (\beta (\gamma - x) + ax+b)) \cdot (x-\gamma)~dx \\
=& 2 \int^{\gamma}_{-1} y(x) x~dx - 2\gamma \int^{\gamma}_{-1} y(x)~dx + 2\int^{\gamma}_{-1} \beta (\gamma - x)^2~dx + 2\int^{\gamma}_{-1} (ax+b) (\gamma - x)~dx \\
=&  2 \int^{\gamma}_{-1} y(x) x~dx - 2\gamma \int^{\gamma}_{-1} y(x)~dx + p(\beta,\gamma) \\
\end{aligned}
\end{equation}
where $p(\beta,\gamma)$ is a third degree polynomial with respect to $\gamma$. We denote by $Y(x)$ the primitive function of $y(x)$, and by $\mathcal{Y}(x)$ the primitive function of $Y(x)$. By applying integration by parts, we have:
\begin{equation}
\int^{\gamma}_{-1} y(x) x~dx = Y(\gamma) \cdot \gamma - (\mathcal{Y}(\gamma) - \mathcal{Y}(-1))
\end{equation}
In particular,
\begin{equation}
\begin{aligned}
Q(\beta,\gamma)=& 2\gamma (Y(\gamma) - Y(-1)) - 2 (Y(\gamma) \cdot \gamma - \mathcal{Y}(\gamma) + \mathcal{Y}(-1)) + p(\beta,\gamma) \\
=& 2\gamma Y(\gamma) - 2\gamma Y(-1) - 2 \gamma Y(\gamma) - 2\mathcal{Y}(\gamma) + 2\mathcal{Y}(-1) + p(\beta,\gamma) \\
=& - 2\mathcal{Y}(\gamma) + [- 2\gamma Y(-1)  + 2\mathcal{Y}(-1) + p(\beta,\gamma)] \\
\end{aligned}
\end{equation}
We note that the function $q(\beta,\gamma) := - 2\gamma Y(-1)  + 2\mathcal{Y}(-1) + p(\beta,\gamma)$ is a third degree polynomial with respect to $\gamma$ (for any fixed $\beta$). In addition, by Eq.~\ref{eq:diffzero}, we have, $Q(0,\gamma)=0$ for any value of $\gamma \in (-1,c_1)$. Hence, $\mathcal{Y}$ is a third degree polynomial over $[-1,c_1]$. In particular, $y$ is a linear function over $[-1,c_1]$, in contradiction. Therefore, there exist values $\gamma \in (-1,c_1)$ and $\beta \in \mathbb{R}$, such that, $f(x)+n(x)$ strictly improves the approximation of $f$. 

\noindent{\bf Sigmoidal activations\quad} Let $y \in \mathbb{Y}$ be a target function that is not a member of $\mathscr{f}$. Let $f \in \mathscr{f}$ be the best approximator of $y$. In particular, $f\neq y$. Since $f$ is a 2-layered neural network, it takes the form:
\begin{equation}
f(x) = \sum^{d}_{i=1} \beta_i \cdot \sigma(\langle \alpha_i, x\rangle + \gamma_i)
\end{equation}
where $\sigma:\mathbb{R}\to \mathbb{R}$ is either $\tanh$ or the sigmoid activation function, $\beta_i,\gamma_i \in \mathbb{R}$ and $\alpha_i \in \mathbb{R}^{m}$.  

We would like to show the existence of a neuron $n(x) = \beta \cdot \sigma(\langle a,x\rangle+b)$, such that, $f+n$ has a smaller approximation error with respect to $y$, compared to $f$. Assume the contrary by contradiction. Then, for any $a\in \mathbb{R}^{m},b \in \mathbb{R}$, we have:
\begin{equation}
\frac{\int_{[-1,1]^m} (y(x) - (\beta \cdot \sigma(\langle a,x\rangle+b) +f(x)))^2~dx }{\partial \beta} \Big|_{\beta=0} = 0
\end{equation}
We denote by $q(x) := y(x)-f(x)$. By differentiating under the integral sign:
\begin{equation}
\begin{aligned}
Q(\beta,a,b) :&= \frac{\int_{[-1,1]^m} (y(x) - (\beta\cdot \sigma(\langle a,x\rangle+b) +f(x)))^2~dx }{\partial \beta} \\
&= -2\int_{[-1,1]^m} \beta \cdot \sigma(\langle a,x\rangle+b)^2~dx + 2\int_{[-1,1]^m} q(x) \cdot \sigma(\langle a,x\rangle+b)~dx
\end{aligned}
\end{equation}
Therefore, since $Q(\beta,a,b)=0$, we have:
\begin{equation}\label{eq:denom}
\beta = \frac{\int_{[-1,1]^m} q(x) \cdot \sigma(\langle a,x\rangle+b)~dx}{\int_{[-1,1]^m} \sigma(\langle a,x\rangle+b)^2~dx}
\end{equation}
Since $\sigma$ is increasing, it is non-zero on any interval, and therefore, the denominator in Eq.~\ref{eq:denom} is strictly positive for all $a\in \mathbb{R}^{m}\setminus\{0\},b\in \mathbb{R}$ and $a=0,b\in \mathbb{R}$, such that, $\sigma(b)\neq 0$. In particular, for all such $a,b$, we have: 
\begin{equation}\label{eq:zerosigma}
\int_{[-1,1]^m} q(x) \cdot \sigma(\langle a,x\rangle+b)~dx = 0 
\end{equation}
By the universal approximation theorem~\cite{Cybenko1989,Hornik1991ApproximationCO}, there exist $\hat{\beta}_j,\hat{b_j} \in \mathbb{R}$ and $\hat{a}_j \in \mathbb{R}^{m}$, such that,
\begin{equation}
f(x) - y(x) = \sum^{\infty}_{j=1} \hat{\beta}_j \cdot \sigma(\langle \hat{a}_j, x\rangle + \hat{b}_j)
\end{equation}
where $\hat{a}_j\in \mathbb{R}^{m}\setminus\{0\},\hat{b}_j\in \mathbb{R}$ and $\hat{a}_j=0,\hat{b}_j\in \mathbb{R}$, such that, $\sigma(\hat{b}_j)\neq 0$. The convergence of the series is uniform over $[-1,1]^m$. In particular, the series $q(x)\cdot \sum^{k}_{j=1} \hat{\beta}_j \cdot \sigma(\langle \hat{a}_j, x\rangle + \hat{b}_j)$ converge uniformly as $k\to \infty$. Therefore, by Eq.~\ref{eq:zerosigma} and the linearity of integration, we have:
\begin{equation}
\int_{[-1,1]^m} q(x) \cdot \sum^{\infty}_{j=1} \hat{\beta}_j \cdot \sigma(\langle \hat{a}_j ,x\rangle +\hat{b}_j)~dx = 0
\end{equation}
This implies that $\int_{[-1,1]^m} q(x)^2~dx = 0$. Since $q$ is a continuous function, it must be the zero function to satisfy this condition. Differently put, $f = y$ in contradiction.
\end{proof}

\subsection{Existence of a continuous selector}

In this section, we prove that for any compact set $\mathbb{Y}' \subset \mathbb{Y}$, if any $y\in \mathbb{Y}'$ cannot be represented as a neural network with $\sigma$ activations, then, there exists a continuous selector $S:\mathbb{Y}' \to \mathbb{R}^{N_{\mathscr{f}}}$ that returns the parameters of a good approximator $f(\cdot;S(y))$ of $y$. Before we provide a formal statement of the proof, we give an informal overview of the main arguments. 

\paragraph{Proof sketch of Lem.~\ref{lem:selection2}} Let $\mathbb{Y}' \subset \mathbb{Y}$ be a compact class of target functions, such that, any $y\in \mathbb{Y}'$ cannot be represented as a neural network with $\sigma$ activations. We recall that, by Lem.~\ref{lem:vlacic}, one can approximate $\sigma$ using a continuous,  identifiability inducing, activation function $\rho:\mathbb{R}\to \mathbb{R}$, up to any error $\epsilon>0$ of our choice. By Assumption~1, for each $y \in \mathbb{Y}$, there exists a unique best function approximator $g(\cdot;\theta_y)\in \mathscr{g}$ of $y$. Here, $\mathscr{g}$ is the class of neural networks of the same architecture as $\mathscr{f}$ except the activations are $\rho$. By Def.~\ref{def:induc}, $\theta_y$ is unique up to isomorphisms, assuming that $g(\cdot;\theta_y)$ is normal (see Def.~\ref{def:normal}).  

In Lem.~\ref{lem:selection} we show that for any compact set $\mathbb{Y}' \subset \mathbb{Y}$, if $g(\cdot;\theta_y)$ is normal for all $y \in \mathbb{Y}'$, then, there exists a continuous selector $S:\mathbb{Y}' \to \mathbb{R}^{N_{\mathscr{g}}}$ that returns the parameters of a best approximator $g(\cdot;S(y))$ of $y$. Therefore, in order to show the existence of $S$, we need to prove that $g(\cdot;\theta_y)$ is normal for all $y \in \mathbb{Y}'$.  

Since any function $y\in \mathbb{Y}'$ cannot be represented as a neural network with $\sigma$ activations, $\inf_{y \in \mathbb{Y}'}\inf_{\theta_f}\|f(\cdot;\theta)-y\|_{\infty}$ is strictly larger than zero (see Lem.~\ref{lem:c2}). In particular, by taking $\rho$ to be close enough to $\sigma$, we can ensure that, $\inf_{y\in \mathbb{Y}'}\inf_{\theta_f}\|g(\cdot;\theta)-y\|_{\infty}$ is also strictly larger than zero. This, together with Assumption~2, imply that $g(\cdot;\theta_y)$ is normal for all $y \in \mathbb{Y}'$ (see Lem.~\ref{lem:normal}). Hence, there exists a continuous selector $S$ for $\mathbb{Y}'$ with respect to the class $\mathscr{g}$. Finally, using Lem.~\ref{lem:c1}, one can show that if $\rho$ is close enough to $\sigma$, $S$ is a good parameter selector for $\mathscr{f}$ as well.

\begin{lemma}\label{lem:subset2}
Let $\rho:\mathbb{R} \to \mathbb{R}$ be a continuous, identifiability inducing, activation function. Let $\mathscr{f}$ be a class of neural networks with $\rho$ activations and $\Theta_{\mathscr{f}} = \mathbb{B}$ be the closed ball in the proof of Lem.~\ref{lem:ball}. Let $\mathbb{Y}$ be a class of normal target functions with respect to $\mathscr{f}$. Then, $M[y;\mathscr{f}] := \arg\min_{\theta \in \mathbb{B}} \| f(\cdot;\theta) - y\|_{\infty}$ is a continuous multi-valued function of $y$.
\end{lemma}

\begin{proof}
Assume by contradiction that $M$ is not continuous. We distinguish between two cases: 
\begin{enumerate}[leftmargin=*]
    \item There exists a sequence $y_n \to y$ and constant $c>0$, such that, 
\begin{equation}
\sup_{\theta \in M[y;\mathscr{f}]} \inf_{\theta \in M[y_{n};\mathscr{f}]} \| \theta_1 - \theta_2\|_2 > c > 0
\end{equation}
    \item There exists a sequence $y_n \to y$ and constant $c>0$, such that, 
\begin{equation}
\sup_{\theta_1 \in M[y_{n};\mathscr{f}]} \inf_{\theta_2 \in M[y;\mathscr{f}]} \| \theta_1 - \theta_2\|_2 > c > 0
\end{equation}
\end{enumerate}

{\bf Case 1:} We denote by $\theta_1$ a member of $M[y;\mathscr{f}]$ that satisfies:
\begin{equation}\label{eq:optone}
\forall n \in \mathbb{N}: \inf_{\theta_2 \in M[y_n;\mathscr{f}]} \| \theta_1 - \theta_2\|_2 > c > 0
\end{equation}
The set $\cup^{\infty}_{n=1} M[y_n;\mathscr{f}] \subset \Theta_{\mathscr{f}}$ is a bounded subset of $\mathbb{R}^{N}$, and therefore by the Bolzano-Weierstrass theorem, for any sequence $\{\theta^{n}_2\}^{\infty}_{n=1}$, such that, $\theta^n_2 \in M[y_n;\mathscr{f}]$, there is a sub-sequence $\{\theta^{n_k}_2\}^{\infty}_{k=1}$ that converges to some $\theta^*_2$. We notice that:
\begin{equation}
\begin{aligned}
\| f(\cdot;\theta^{n_k}_2) - y_{n_{k}}\|_{\infty} = \min_{\theta \in \Theta_{\mathscr{f}}} \|  f(\cdot;\theta) - y_{n_{k}}\|_{\infty} 
= F(y_{n_{k}})
\end{aligned}
\end{equation}
In addition, by the continuity of $F$, we have: $\lim\limits_{k\to \infty} F(y_{n_{k}}) = F(y)$. By Lem.~\ref{lem:cont1}, we have:
\begin{equation}
\| f(\cdot;\theta^*_2) - y\|_{\infty} = F(y)
\end{equation}
This yields that $\theta^*_2$ is a member of $M[y;\mathscr{f}]$. Since $f_y := \arg\min_{f \in \mathscr{f}} \| f - y\|_{\infty}$ is unique and normal, by the identifiability hypothesis, there is a function $\pi \in \Pi$, such that, $\pi(\theta^*_2) = \theta_1$. Since the function $\pi$ is continuous
\begin{equation}
\begin{aligned}
\lim_{k \to \infty} \|\pi (\theta^{n_k}_2) - \theta_1 \|_2  = \lim_{k \to \infty} \|\pi(\theta^{n_k}_2) - \pi(\theta^*_2)\|_2 = 0
\end{aligned}
\end{equation}
We notice that $\pi(\theta^{n_k}_2) \in M[y_{n_{k}};\mathscr{f}]$. Therefore, we have:
\begin{equation}
\lim_{k \to \infty} \inf_{\theta_2 \in M[y_{n_{k}};\mathscr{f}]} \|\theta_1 - \theta_2\| = 0
\end{equation}
in contradiction to Eq.~\ref{eq:optone}. 

{\bf Case 2:} Let $\theta^n_1\in M[y_n;\mathscr{f}]$ be a sequence, such that, 
\begin{equation}\label{eq:largerc}
\inf_{\theta_2 \in M[y;\mathscr{f}]} \|\theta^n_1 - \theta_2 \|_{\infty} > c
\end{equation}
The set $\cup^{\infty}_{n=1} M[y_n;\mathscr{f}] \subset \Theta_{\mathscr{f}}$ is a bounded subset of $\mathbb{R}^{N}$, and therefore by the Bolzano-Weierstrass theorem, there is a sub-sequence $\theta^{n_k}_1$ that converges to some vector $\theta_0$. The function $\|f(\cdot;\theta) - y \|_{\infty}$ is continuous with respect to $\theta$ and $y$. Therefore, 
\begin{equation}
\begin{aligned}
\lim_{k \to \infty} \min_{\theta \in \Theta_{\mathscr{f}}}\|f(\cdot;\theta) - y_{n_k} \|_{\infty} 
= \lim_{k \to \infty} \|f(\cdot;\theta^{n_k}_1) - y_{n_k} \|_{\infty} = \|f(\cdot;\theta_0) - y \|_{\infty}
\end{aligned}
\end{equation}
By Lem.~\ref{lem:cont}, $\|f(\cdot;\theta_0) - y \|_{\infty} = \min_{\theta \in \Theta_{\mathscr{f}}} \|f(\cdot;\theta) - y \|_{\infty}$. In particular, $\theta_0 \in M[y;\mathscr{f}]$, in contradiction to Eq.~\ref{eq:largerc}.
\end{proof}

\begin{lemma}\label{lem:selection}
Let $\rho:\mathbb{R} \to \mathbb{R}$ be a continuous, identifiability inducing, activation function. Let $\mathscr{f}$ be a class of neural networks with $\rho$ activations and $\Theta_{\mathscr{f}} = \mathbb{B}$ be the closed ball in the proof of Lem.~\ref{lem:ball}. Let $\mathbb{Y}$ be a compact class of normal target functions with respect to $\mathscr{f}$. Then, there is a continuous selector $S:\mathbb{Y} \to \Theta_{\mathscr{f}}$, such that, $S(y) \in M[y;\mathscr{f}]$.
\end{lemma}

\begin{proof}
Let $y_0$ be a member of $\mathbb{Y}$. We notice that $M[y_0;\mathscr{f}]$ is a finite set. We denote its members by: $M[y_0;\mathscr{f}] = \{\theta^0_1,\dots,\theta^0_k\}$. Then, we claim that there is a small enough $\epsilon := \epsilon(y_0)>0$ (depending on $y_0$), such that, $S$ that satisfies $S(y_0) = \theta^0_1$ and $S(y) = \arg\min_{\theta \in M[y;\mathscr{f}]} \| \theta - \theta_0 \|_2$ for all $y \in \mathbb{B}_{\epsilon}(y_0)$, is continuous in $\mathbb{B}_{\epsilon}(y_0)$. The set $\mathbb{B}_{\epsilon}(y_0) := \{y \;\vert\; \|y-y_0\|_{\infty} < \epsilon\}$ is the open ball of radius $\epsilon$ around $y_0$. We denote 
\begin{equation}\label{eq:c}
c := \min_{\pi_1\neq \pi_2 \in \Pi} \|\pi_1 \circ S(y_0) - \pi_2 \circ S(y_0)\|_2 > 0  
\end{equation}
This constant exists since $\Pi$ is a finite set of transformations and $\mathbb{Y}$ is a class of normal functions. In addition, we select $\epsilon$ to be small enough to suffice that:
\begin{equation}
\max_{y \in \mathbb{B}_{\epsilon}(y_0)}\|S(y) - S(y_0) \|_2 < c/4
\end{equation}
Assume by contradiction that there is no such $\epsilon$. Then, for each $\epsilon_n = 1/n$ there is a function $y_n \in \mathbb{B}_{\epsilon_n}(y_0)$, such that, 
\begin{equation}
\|S(y) - S(y_0) \|_2 \geq c/4
\end{equation}
Therefore, we found a sequence $y_n \to y_0$ that satisfies: 
\begin{equation}
M[y_n;\mathscr{f}] \not\to M[y_0;\mathscr{f}]
\end{equation}
in contradiction to the continuity of $M$.

For any given $y_1,y_2 \in \mathbb{B}_{\epsilon}(y_0)$ and $\pi_1 \neq \pi_2 \in \Pi$, by the triangle inequality, we have:
\begin{equation}
\begin{aligned}
\|\pi_1 \circ S(y_1) - \pi_2 \circ S(y_2)\|_2 \geq &\|\pi_1 \circ S(y_0) - \pi_2 \circ S(y_2)\|_2 - \|\pi_1 \circ S(y_1) - \pi_1 \circ S(y_0) \|_2  \\
\geq & \|\pi_1 \circ S(y_0) - \pi_2 \circ S(y_0)\|_2 - \|\pi_1 \circ S(y_1) - \pi_1 \circ S(y_0) \|_2  \\
&- \|\pi_2 \circ S(y_0) - \pi_2 \circ S(y_2) \|_2 \\
= & \|\pi_1 \circ S(y_0) - \pi_2 \circ S(y_0)\|_2 - \|S(y_1) - S(y_0) \|_2  \\
&- \|S(y_0) - S(y_2) \|_2 \\
\geq & c-2c/4 > c/2 
\end{aligned}
\end{equation}
In particular, $\|\pi \circ S(y_1) - S(y_2)\|_2 > c/2$ for every $\pi \neq \textnormal{Id}$. 

Since $M$ is continuous, for any sequence $y_n \to y \in \mathbb{B}_{\epsilon}(y_0)$, there are $\pi_n \in \Pi$, such that:
\begin{equation}
\lim_{n \to \infty} \pi_n\circ S(y_n) = S(y) 
\end{equation}
Therefore, by the above inequality, we address that for any large enough $n$, $\pi_n=\textnormal{Id}$. In particular, for any sequence $y_n \to y$, we have:
\begin{equation}
\lim_{n \to \infty} S(y_n) = S(y) 
\end{equation}
This implies that $S$ is continuous in any $y \in \mathbb{B}_{\epsilon}(y_0)$.

We note that  $\{\mathbb{B}_{\epsilon(y_0)}(y_0)\}_{y_0 \in \mathbb{Y}}$ is an open cover of $\mathbb{Y}$. In particular, since $\mathbb{Y}$ is compact, there is a finite sub-cover $\{C_i\}^{T}_{i=1}$ of $\mathbb{Y}$. In addition, we denote by $\{c_i\}^{T}_{i=1}$ the corresponding constants in Eq.~\ref{eq:c}. Next, we construct the continuous function $S$ inductively. We denote by $S_i$ the locally continuous function that corresponds to $C_i$. For a given pair of sets $C_{i_1}$ and $C_{i_2}$ that intersect, we would like to construct a continuous function over $C_{i_1} \cup C_{i_2}$. First, we would like to show that there is an isomorphism $\pi$, such that, $\pi\circ S_{i_2}(y) = S_{i_1}(y)$ for all $y \in C_{i_1} \cap C_{i_2}$. Assume by contradiction that there is no such $\pi$. Then, let $y_1 \in C_{i_1} \cap C_{i_2}$ and $\pi_1$, such that, $\pi_1\circ S_{i_2}(y_1) = S_{i_1}(y_1)$. We denote by $y_2 \in C_{i_1} \cap C_{i_2}$ a member, such that, $\pi_1\circ S_{i_2}(y_2) \neq S_{i_1}(y_2)$. Therefore, we take a isomorphism $\pi_2 \neq \pi_1$, that satisfies $\pi_2\circ S_{i_2}(y_2) = S_{i_1}(y_2)$. We note that:
\begin{equation}
\| \pi_1\circ S_{i_2}(y_1) - \pi_2\circ S_{i_2}(y_2) \|_2 > \max\{c_{i_1},c_{i_2}\}/2
\end{equation}
on the other hand:
\begin{equation}
\begin{aligned}
\| \pi_1\circ S_{i_2}(y_1) - \pi_2\circ S_{i_2}(y_2) \|_2 = \| S_{i_1}(y_1) - S_{i_1}(y_2)\|_2 < c_{i_1}/4
\end{aligned}
\end{equation}
in contradiction. 

Hence, let $\pi$ be such isomorphism. To construct a continuous function over $C_{i_1} \cup C_{i_2}$ we proceed as follows. First, we replace $S_{i_2}$ with $\pi \circ S_{i_2}$ and define a selection function $S_{i_1,i_2}$ over $C_{i_1} \cup C_{i_2}$ to be:
\begin{equation}
S_{i_1,i_2}(y) := 
\begin{cases}
       S_{i_1}(y) &\quad\text{if }, y \in C_{i_1}\\
       \pi \circ S_{i_2}(y) &\quad\text{if }, y \in C_{i_2} \\
     \end{cases}
\end{equation}
Since each one of the functions $S_{i_1}$ and $\pi \circ S_{i_2}$ are continuous, they conform on $C_{i_1} \cap C_{i_2}$ and the sets $C_{i_1}$ and $C_{i_2}$ are open, $S_{i_1,i_2}$ is continuous over $C_{i_1} \cup C_{i_2}$. We define a new cover $(\{C_i\}^{T}_{i=1} \setminus \{C_{i_1},C_{i_2}\}) \cup \{C_{i_1} \cup C_{i_2}\}$ of size $T-1$ with locally continuous selection functions $S'_1, \dots, S'_{T-1}$. By induction, we can construct $S$ over $\mathbb{Y}$.
\end{proof}

\begin{lemma}\label{lem:c2}
Let $\mathscr{f}$ be a class of neural networks with a continuous activation function $\sigma$. Let $\mathbb{Y}$ be a compact class of target functions. Assume that any $y \in \mathbb{Y}$ cannot be represented as a neural network with $\sigma$ activations. Then, 
\begin{equation}
\inf_{y \in \mathbb{Y}}\inf_{\theta \in \Theta_{\mathscr{f}}} \| f(\cdot;\theta) - y\|_{\infty} > c_2 
\end{equation}
for some constant $c_2>0$.
\end{lemma}

\begin{proof}
Assume by contradiction that:
\begin{equation}
\inf_{y \in \mathbb{Y}}\inf_{\theta \in \Theta_{\mathscr{f}}} \| f(\cdot;\theta) - y\|_{\infty} = 0  
\end{equation}
Then, there is a sequence $y_n \in \mathbb{Y}$, such that:
\begin{equation}
\inf_{\theta \in \Theta_{\mathscr{f}}} \| f(\cdot;\theta) - y_{n}\|_{\infty}  \to 0
\end{equation}
Since $\mathbb{Y}$ is compact, there exists a converging sub-sequence $y_{n_k} \to y_0 \in \mathbb{Y}$. By Lem.~\ref{lem:cont}, we have: 
\begin{equation}
\inf_{\theta \in \Theta_{\mathscr{f}}} \| f(\cdot;\theta) - y_{0} \|_{\infty} = 0 
\end{equation}
This is in contradiction to the assumption that any $y\in \mathbb{Y}$ cannot be represented as a neural network with $\sigma$ activations.
\end{proof}

\begin{lemma}\label{lem:bounded2}
Let $\mathscr{f}$ be a class of neural networks with a continuous activation function $\sigma$. Let $\mathbb{Y}$ be a compact class of target functions. Assume that any $y \in \mathbb{Y}$ cannot be represented as a neural network with $\sigma$ activations. Then, there exists a closed ball $\mathbb{B}$ around $0$ in the Euclidean space $\mathbb{R}^{N_{\mathscr{f}}}$, such that:
\begin{equation}
\min_{\theta \in \mathbb{B}} \|f(\cdot;\theta) - y\|_{\infty} \leq 2\inf_{\theta \in \Theta_{\mathscr{f}}}\|f(\cdot;\theta) - y\|_{\infty}
\end{equation}
\end{lemma}

\begin{proof}
Let $c_2>0$ be the constant from Lem.~\ref{lem:c2}. By Lem.~\ref{lem:c2} and Lem.~\ref{lem:cont}, $f_y$ is continuous over the compact set $\mathbb{Y}$. Therefore, there is a small enough $\delta>0$, such that, for any $y_1,y_2 \in \mathbb{Y}$, such that, $\|y_1-y_2\|_{\infty}<\delta$, we have: $\|f_{y_1} - f_{y_2}\|_{\infty}<c_2/2$. For each $y \in \mathbb{Y}$ we define $B(y) := \{y' \;\vert\; \|y-y'\|_{\infty} < \min\{c_2/2,\delta\}\}$. The sets $\{B(y)\}_{y\in \mathbb{Y}}$ form an open cover to $\mathbb{Y}$. Since $\mathbb{Y}$ is a compact set, it has a finite sub-cover $\{B(y_1),\dots,B(y_k)\}$. For each $y' \in B(y_i)$, we have:
\begin{equation}
\begin{aligned}
\|f_{y_i} - y'\|_{\infty} &\leq \|f_{y_i} - f_{y'}\|_{\infty} + \|f_{y'} - y'\|_{\infty} \\
&\leq c_2/2 + \|f_{y'} - y'\|_{\infty}  \\
&\leq 2\|f_{y'} - y'\|_{\infty} 
\end{aligned}
\end{equation}
Therefore, if we take $H = \{\theta_i\}^{k}_{i=1}$ for $\theta_i$, such that, $f(\cdot;\theta_i)=f_{y_i}$, we have:
\begin{equation}
\min_{i\in [n]} \|f(\cdot;\theta_i) - y\|_{\infty} \leq 2\inf_{\theta \in \Theta_{\mathscr{f}}}\|f(\cdot;\theta) - y\|_{\infty}
\end{equation}
In particular, if we take $\mathbb{B}$ to be the closed ball around $0$ that contains $H$, we have the desired.
\end{proof}

\begin{lemma}\label{lem:c1}
Let $\sigma:\mathbb{R} \to \mathbb{R}$ be a $L$-Lipschitz continuous activation function. Let $\mathscr{f}$ be a class of neural networks with $\sigma$ activations. Let $\mathbb{Y}$ be a compact class of normal target functions with respect to $\mathscr{f}$. Let $\rho$ be an activation function, such that, $\|\sigma-\rho\|_{\infty} < \delta$. Let $\mathbb{B} = \mathbb{B}_1 \cup \mathbb{B}_2$ be the closed ball around $0$, where $\mathbb{B}_1$ is be the closed ball in the proof of Lem.~\ref{lem:ball} and $\mathbb{B}_2$ is the ball from Lem.~\ref{lem:bounded2}. In addition, let $\mathscr{g}$ be the class of neural networks of the same architecture as $\mathscr{f}$ except the activations are $\rho$. Then, for any $\theta \in \mathbb{B}$, we have:
\begin{equation}\label{eq:activations}
\| f(\cdot;\theta) - g(\cdot;\theta)\|_{\infty} \leq c_1 \cdot \delta
\end{equation}
for some constant $c_1>0$ independent of $\delta$.
\end{lemma}

\begin{proof}
We prove by induction that for any input $x \in \mathcal{X}$ the outputs the $i$'th layer of $f(\cdot;\theta)$ and $g(\cdot;\theta)$ are $\mathcal{O}(\delta)$-close to each other. 

{\bf Base case:} we note that:
\begin{equation}
\begin{aligned}
\| \sigma(W^1 \cdot x + b^1) - \rho(W^1 \cdot x + b^1) \|_1 
&\leq \sum^{h_2}_{i=1} \Big\vert \sigma(\langle W^1_i, x\rangle + b^1_i) - \rho(\langle W^1_i, x\rangle + b^1_i) \Big\vert \\
&\leq h_2 \cdot \delta =: c^1 \cdot \delta
\end{aligned}
\end{equation}
Hence, the first layer's activations are $\mathcal{O}(\delta)$-close to each other. 

{\bf Induction step:} assume that for any two vectors of activations $x_1$ and $x_2$ in the $i$'th layer of the neural networks, we have:
\begin{equation}
\|x_1 - x_2\|_1 \leq c^i \cdot \delta
\end{equation}
By the triangle inequality:
\begin{equation}
\begin{aligned}
&\|\sigma(W^{i+1} \cdot x_1 + b^{i+1}) - \rho(W^{i+1} x_2 + b^{i+1})\|_1 \\ 
\leq& \|\sigma(W^{i+1} \cdot x_1 + b^{i+1}) - \sigma(W^{i+1} x_2 + b^{i+1}) \|_1 \\
&+ \|\sigma(W^{i+1} x_2 + b^{i+1}) - \rho(W^{i+1} x_2 + b^{i+1}) \|_1 \\
\leq& L \cdot \| (W^{i+1} \cdot x_1 + b^{i+1}) - (W^{i+1} x_2 + b^{i+1}) \|_1 \\
&+ \sum^{h_{i+2}}_{j=1} \vert \sigma(\langle W^{i+1}_j, x\rangle + b^{i+1}_j) - \rho(\langle W^{i+1}_j, x\rangle + b^{i+1}_j) \vert \\
=& L \cdot \| W^{i+1} (x_1 - x_2) \|_1 + h_{i+2}\cdot \delta \\
\leq& L \cdot \|W^{i+1}\|_1 \cdot \| x_1 - x_2 \|_1 + h_{i+2} \cdot \delta \\
\leq& L \cdot \|W^{i+1}\|_1 \cdot c^i \cdot \delta + h_{i+2} \cdot \delta \\
\leq& (h_{i+2} + L \cdot \|W^{i+1}\|_1 \cdot c^i ) \cdot \delta \\
\end{aligned}
\end{equation}
Since $\theta \in \mathbb{B}$ is bounded, each $\|W^{i+1}\|_1$ is bounded (for all $i \leq k$ and $\theta$). Hence, Eq.~\ref{eq:activations} holds for some constant $c_1>0$ independent of $\delta$.
\end{proof}

\begin{lemma}\label{lem:c3}
Let $\sigma:\mathbb{R} \to \mathbb{R}$ be a $L$-Lipschitz continuous activation function. Let $\mathscr{f}$ be a class of neural networks with $\sigma$ activations. Let $\mathbb{Y}$ be a compact class of target functions. Assume that any $y \in \mathbb{Y}$ cannot be represented as a neural network with $\sigma$ activations. Let $\rho$ be an activation function, such that, $\|\sigma-\rho\|_{\infty} < \delta$. Let $\mathbb{B}$ be the closed ball from Lem.~\ref{lem:c1}. In addition, let $\mathscr{g}$ be the class of neural networks of the same architecture as $\mathscr{f}$ except the activations are $\rho$. Then, for any $y \in \mathbb{Y}$, we have:
\begin{equation}
\Big\vert \min_{\theta \in \mathbb{B}} \| f(\cdot;\theta) - y\|_{\infty} - \min_{\theta \in \mathbb{B}} \| g(\cdot;\theta) - y\|_{\infty} \Big\vert \leq c_1 \cdot \delta 
\end{equation}
for $c_1$ from Lem.~\ref{lem:c1}.
\end{lemma}

\begin{proof}
By Lem.~\ref{lem:c1}, for all $\theta \in \mathbb{B}$, we have:
\begin{equation}
\| f(\cdot;\theta) - y\|_{\infty} \leq \| g(\cdot;\theta) - y\|_{\infty} + c_1 \cdot \delta 
\end{equation}
In particular,
\begin{equation}
\min_{\theta \in \mathbb{B}} \| f(\cdot;\theta) - y\|_{\infty} \leq \min_{\theta \in \mathbb{B}}\| g(\cdot;\theta) - y\|_{\infty} + c_1 \cdot \delta 
\end{equation}
By a similar argument, we also have:
\begin{equation}
\min_{\theta \in \mathbb{B}}\| g(\cdot;\theta) - y\|_{\infty} \leq \min_{\theta \in \mathbb{B}} \| f(\cdot;\theta) - y\|_{\infty} + c_1 \cdot \delta 
\end{equation}
Hence,
\begin{equation}
\Big\vert \min_{\theta \in \mathbb{B}} \| f(\cdot;\theta) - y\|_{\infty} - \min_{\theta \in \mathbb{B}} \| g(\cdot;\theta) - y\|_{\infty} \Big\vert \leq c_1 \cdot \delta 
\end{equation}
\end{proof}

\begin{lemma}\label{lem:selection2}
Let $\sigma:\mathbb{R} \to \mathbb{R}$ be a $L$-Lipschitz continuous activation function. Let $\mathscr{f}$ be a class of neural networks with $\sigma$ activations. Let $\mathbb{Y}$ be a compact set of target functions. Assume that any $y \in \mathbb{Y}$ cannot be represented as a neural network with $\sigma$ activations. Then, for every $\hat{\epsilon}>0$ there is a continuous selector $S:\mathbb{Y} \to \Theta_{\mathscr{f}}$, such that, for all $y \in \mathbb{Y}$, we have:
\begin{equation}
\|f(\cdot ;S(y)) - y \|_{\infty} \leq 2\inf_{\theta \in \Theta_{\mathscr{f}}} \|f(\cdot ;\theta) - y \|_{\infty} + \hat{\epsilon}
\end{equation}
\end{lemma}

\begin{proof}
By Lem.~\ref{lem:vlacic}, there exists a meromorphic function $\rho : D \to \mathbb{C}$, $\mathbb{R} \subset D$, $\rho(\mathbb{R}) \subset \mathbb{R}$ that is an identifiability inducing function, such that, $\|\sigma-\rho\|_{\infty} < \frac{1}{2c_2}\min(\hat{\epsilon},c_1) =: \delta$, where $c_1$ and $c_2$ are the constants in Lems.~\ref{lem:c1} and~\ref{lem:c2}. Since $\rho(\mathbb{R}) \subset \mathbb{R}$, and it is a meromorphic over $D$, it is continuous over $\mathbb{R}$ (the poles of $\rho$ are not in $\mathbb{R}$). We note that by Lems.~\ref{lem:c1} and~\ref{lem:c3}, for any $y \in \mathbb{Y}$, we have:
\begin{equation}
\min_{\theta \in \mathbb{B}} \| g(\cdot;\theta) - y \|_{\infty} > c_2 - c_1 \cdot \delta > 0
\end{equation}
where $\mathbb{B}$ is the ball from Lem.~\ref{lem:c1}. Therefore, by Lem.~\ref{lem:normal}, each $y\in \mathbb{Y}$ is normal with respect to the class $\mathscr{g}$. Hence, by Thm.~\ref{lem:selection}, there is a continuous selector $S:\mathbb{Y} \to \mathbb{B}$, such that,
\begin{equation}
\|g(\cdot ;S(y)) - y \|_{\infty} = \min_{\theta \in \mathbb{B}}\|g(\cdot ;\theta) - y \|_{\infty}
\end{equation}
By Lem.~\ref{lem:c3}, we have:
\begin{equation}\label{eq:c3}
\Big\vert \min_{\theta \in \mathbb{B}} \| f(\cdot;\theta) - y\|_{\infty} - \|g(\cdot ;S(y)) - y \|_{\infty} \Big\vert \leq c_1 \cdot \delta 
\end{equation}
By the triangle inequality:
\begin{equation}
\begin{aligned}
&\Big\vert \min_{\theta \in \mathbb{B}} \| f(\cdot;\theta) - y\|_{\infty} - \|f(\cdot ;S(y)) - y \|_{\infty} \Big\vert \\
\leq&  \Big\vert \|f(\cdot ;S(y)) - y \|_{\infty} - \|g(\cdot ;S(y)) - y \|_{\infty} \Big\vert + \Big\vert \min_{\theta \in \mathbb{B}} \| f(\cdot;\theta) - y\|_{\infty} - \|g(\cdot ;S(y)) - y \|_{\infty} \Big\vert \\
\end{aligned}
\end{equation}
By Eq.~\ref{eq:c3} and Lem.~\ref{lem:c1}, we have:
\begin{equation}
\Big\vert \min_{\theta \in \mathbb{B}} \| f(\cdot;\theta) - y\|_{\infty} - \|f(\cdot ;S(y)) - y \|_{\infty} \Big\vert \leq 2c_1 \cdot \delta
\end{equation}
Since $\delta < \hat{\epsilon}/2c_2$, we obtain the desired inequality:
\begin{equation}
\begin{aligned}
\|f(\cdot ;S(y)) - y \|_{\infty} &\leq \min_{\theta \in \mathbb{B}} \| f(\cdot;\theta) - y\|_{\infty} + \hat{\epsilon}\\
&\leq 2\min_{\theta \in \Theta_{\mathscr{f}}} \| f(\cdot;\theta) - y\|_{\infty} + \hat{\epsilon}\\
\end{aligned}
\end{equation}
\end{proof}

\subsection{Proof of Thm.~\ref{thm:generallower}}

Before we provide a formal statement of the proof, we introduce an informal outline of it.

\paragraph{Proof sketch of Thm.~\ref{thm:generallower}} In Lem.~\ref{lem:selection2} we showed that for a compact class $\mathbb{Y}$ of target functions that cannot be represented as neural networks with $\sigma$ activations, there is a continuous selector $S(y)$ of parameters, such that,
\begin{equation}
\|f(\cdot;S(y)) - y\|_{\infty} \leq 3 \inf_{\theta \in \Theta_{\mathscr{f}}}\| f(\cdot;\theta)-y\|_{\infty}
\end{equation}
Therefore, in this case, we have: $d_N(\mathscr{f};\mathbb{Y}) = \Theta(\tilde{d}_N(\mathscr{f};\mathbb{Y}))$. As a next step, we would like to apply this claim on $\mathbb{Y} := \mathcal{W}_{r,m}$ and apply the lower bound of $\tilde{d}_N(\mathscr{f};\mathcal{W}_{r,m}) = \Omega(N^{-r/m})$ to lower bound $d_N(\mathscr{f};\mathbb{Y})$. However, both of the classes $\mathscr{f}$ and $\mathbb{Y}$ include constant functions, and therefore, we have: $\mathscr{f} \cap \mathbb{Y} \neq \emptyset$. Hence, we are unable to assume that any $y \in \mathbb{Y}$ cannot be represented as a neural network with $\sigma$ activations. 

To solve this issue, we consider a ``wide'' compact subset $\mathbb{Y}' = \mathcal{W}^{\gamma}_{r,m}$ of $\mathcal{W}_{r,m}$ that does not include any constant functions, but still satisfies $\tilde{d}_N(\mathscr{f};\mathcal{W}^{\gamma}_{r,m}) = \Omega(N^{-r/m})$. Then, assuming that any non-constant function $y \in \mathcal{W}_{r,m}$ cannot be represented as a neural network with $\sigma$ activations, implies that any $y \in \mathcal{W}^{\gamma}_{r,m}$ cannot be represented as a neural network with $\sigma$ activations. In particular, by Lem.~\ref{lem:selection2}, we obtain the desired lower bound: $d_N(\mathscr{f};\mathcal{W}_{r,m})\geq d_N(\mathscr{f};\mathcal{W}^{\gamma}_{r,m}) = \Theta(\tilde{d}_N(\mathscr{f};\mathcal{W}^{\gamma}_{r,m})) = \Omega(N^{-r/m})$.

For this purpose, we provide some technical notations. For a given function $f:[-1,1]^m \to \mathbb{R}$, we denote:
\begin{equation}
\|h\|^{s,*}_{r} := \sum_{1 \leq \vert \mathbf{k}\vert_1 \leq r} \|D^{\mathbf{k}} h\|_{\infty}
\end{equation}
In addition, for any $0\leq \gamma_1 < \gamma_2 < \infty$, we define:
\begin{equation}
\mathcal{W}^{\gamma_1,\gamma_2}_{r,m} := \left\{f:[-1,1]^m\to \mathbb{R} \mid \textnormal{$f$ is $r$-smooth and } \|f\|^s_{r} \leq \gamma_2 \textnormal{ and } \|f\|^{s,*}_{r} \geq \gamma_1 \right\}
\end{equation}
Specifically, we denote, $\mathcal{W}^{\gamma_1}_{r,m}$ when $\gamma_2=1$. We notice that this set is compact, since it is closed and subset to the compact set $\mathcal{W}_{r,m}$ (see~\cite{sobolev}).

Next, we would like to produce a lower bound for the $N$-width of $\mathcal{W}^{\gamma}_{r,m}$. In~\cite{devore,Pinkus1985NWidthsIA}, in order to achieve a lower bound for the $N$-width of $\mathcal{W}_{r,m}$, two steps are taken. First, they prove that for any $K \subset L^{\infty}([-1,1]^m)$, we have: $\tilde{d}_{N}(K) \geq b_{N}(K)$. Here, $b_N(K) :=\sup_{X_{N+1}} \sup\left\{\rho \;\vert\; \rho \cdot U(X_{N+1}) \subset K \right\}$ is the Bernstein $N$-width of $K$. The supremum is taken over all $N + 1$ dimensional linear subspaces $X_{N+1}$ of $L^{\infty}([-1,1]^m)$ and $U(X) := \{f \in X \;\vert\; \|f\|_{\infty} \leq 1\}$ stands for the unit ball of $X$. As a second step, they show that the Bernstein $N$-width of $\mathcal{W}_{r,m}$ is larger than $\Omega(N^{-r/m})$. 

Unfortunately, in the general case, Bernstein's $N$-width is very limited in its ability to estimate the nonlinear $N$-width. When considering a set $K$ that is not centered around $0$, Bernstein's $N$-width can be arbitrarily smaller than the actual nonlinear $N$-width of $K$. For example, if all of the members of $K$ are distant from $0$, then, the Bernstein's $N$-width of $K$ is zero but the nonlinear $N$-width of $K$ that might be large. Specifically, the Bernstein $N$-width of $\mathcal{W}^{\gamma}_{r,m}$ is small even though intuitively, this set should have a similar width as the standard Sobolev space (at least for a small enough $\gamma>0$). Therefore, for the purpose of measuring the width of $\mathcal{W}^{\gamma}_{r,m}$, we define the extended Bernstein $N$-width of a set $K$,
\begin{equation}
\tilde{b}_{N}(K) := \sup_{X_{N+1}} \sup\left\{\rho \;\big\vert\;\exists \beta<\rho \textnormal{ s.t } \rho \cdot U(X_{N+1}) \setminus \beta \cdot U(X_{N+1}) \subset K \right\}
\end{equation}
with the supremum taken over all $N + 1$ dimensional linear subspaces $X_{N+1}$ of $L^{\infty}([-1,1]^m)$. 

The following lemma extends Lem.~3.1 in~\cite{devore} and shows that the extended Bernstein $N$-width of a set $K$ is a lower bound of the nonlinear $N$-width of $K$.
\begin{lemma}
Let $K \subset L^{\infty}([-1,1]^m)$. Then, $\tilde{d}_{N}(K) \geq \tilde{b}_{N}(K)$.
\end{lemma}

\begin{proof}
The proof is based on the proof of Lem.~3.1 in~\cite{devore}. For completeness, we re-write the proof with minor modifications. Let $\rho < \tilde{b}_{N}(K)$ and let $X_{N+1}$ be an $N + 1$ dimensional subspace of $L^{\infty}([-1,1]^m)$, such that, there exists $0<\beta<\rho$ and $[\rho \cdot U(X_{N+1}) \setminus \beta \cdot U(X_{N+1})] \subset K$. If $\mathscr{f}(\cdot;\theta)$ is class of functions with $N_{\mathscr{f}} = N$ parameters and $S(y)$ is any continuous selection for $K$, such that,
\begin{equation}
\alpha := \sup_{y \in K}\|f(\cdot;S(y)) - y \|_{\infty}
\end{equation}
we let $\hat{S}(y) := S(y) - S(-y)$. We notice that, $\hat{S}(y)$ is an odd continuous mapping of $\partial (\rho \cdot U(X_{N+1}))$ into $\mathbb{R}^N$. Hence, by the Borsuk-Ulam antipodality theorem~\cite{Borsuk1933,LyuShn47} (see also~\cite{dodsonParker}), there is a function $y_0$ in $\partial(\rho \cdot U(X_{N+1}))$ for which $\hat{S}(y_0) = 0$, i.e. $S(-y_0) = S(y_0)$. We write 
\begin{equation}
2y_0 = (y_0 - \mathscr{f}(\cdot;S(y_0)) - (-y_0 - \mathscr{f}(\cdot;S(-y_0))
\end{equation}
and by the triangle inequality:
\begin{equation}
2\rho = 2\|y_0\|_{\infty} \leq  \| y_0 - \mathscr{f}(\cdot;S(y_0)\|_{\infty} + \| -y_0 - \mathscr{f}(\cdot;S(-y_0)\|_{\infty}
\end{equation}
It follows that one of the two functions $y_0$, $-y_0$ are approximated by $\mathscr{f}(\cdot;S(y_0))$ with an error
$\geq \rho$. Therefore, we have: $\alpha \geq \rho$. Since the lower bound holds uniformly for all continuous selections $S$, we have: $\tilde{d}_{N}(K) \geq \rho$. 
\end{proof}

\begin{lemma}\label{lem:lowerBoundGamma}
Let $\gamma \in (0,1)$ and $r,m,N \in \mathbb{N}$. We have:
\begin{equation}
\tilde{d}_N(\mathcal{W}^{\gamma}_{r,m}) \geq C \cdot N^{-r/m}
\end{equation}
for some constant $C>0$ that depends only on $r$.
\end{lemma}

\begin{proof} 
Similar to the proof of Thm.~4.2 in~\cite{devore} with additional modifications. We fix the integer $r$ and let $\phi$ be a $C^{\infty}(\mathbb{R}^m)$ function which is one on the cube $[1/4, 3/4]^m$ and vanishes outside of $[-1,1]^m$. Furthermore, let $C_0$ be such that $1 < \| D^{\mathbf{k}} \phi\|_{\infty} < C_0$, for all $\vert \mathbf{k}\vert < r$. With no loss of generality, we consider integers $N$ of the form $N=d^m$ for some positive integer $d$ and we let $Q_1,\dots,Q_N$ be the partition of $[-1,1]^m$ into closed cubes of side length $1/d$. Then, by applying a linear change of variables which takes $Q_j$ to $[-1,1]^m$, we obtain functions $\phi_1,\dots,\phi_N$ with $\phi_j$ supported on $Q_j$, such that:
\begin{equation}\label{eq:boundphi}
\forall \mathbf{k} \textnormal{ s.t } \vert \mathbf{k}\vert \leq r : d^{\vert \mathbf{k}\vert} \leq \|D^{\mathbf{k}}\phi_j\|_{\infty} \leq C_0 \cdot d^{\vert \mathbf{k}\vert} 
\end{equation}
We consider the linear space $X_{N}$ of functions $\sum^{N}_{j=1} c_j \cdot \phi_j$ spanned by the functions $\phi_1,\dots,\phi_N$. Let $y = \sum^{N}_{j=1} c_j \cdot \phi_i$. By Lem.~4.1 in~\cite{devore}, for $p=q=\infty$, we have:
\begin{equation}
\|y\|^{s}_{r} \leq C_1 \cdot N^{r/m} \cdot \max_{j\in [N]} |c_j| 
\end{equation}
for some constant $C_1>0$ depending only on $r$. By definition, for any $x\in Q_j$, we have: $y(x) = c_j \cdot \phi_j(x)$. In particular,
\begin{equation}
\|y\|_{\infty} = \max_{j\in [N]}\max_{x\in Q_j} |c_j| \cdot \| \phi_j(x)\|_{\infty}
\end{equation}
Therefore, by Eq.~\ref{eq:boundphi}, we have: 
\begin{equation}\label{eq:y}
\max_{j\in [N]} |c_j| \leq \|y\|_{\infty} \leq C_0 \cdot \max_{j\in [N]} |c_j|
\end{equation}
Hence,
\begin{equation}
\|y\|^{s}_{r} \leq C_1 \cdot N^{r/m} \cdot \|y\|_{\infty}
\end{equation}
Then, by taking $\rho := C_1^{-1} \cdot N^{-r/m}$, any $y\in \rho\cdot U(X_N)$ satisfies $\|y\|^{s}_{r} \leq 1$. Again, by Lem.~4.1 and Eq.~\ref{eq:boundphi}, we also have: 
\begin{equation}
\|y\|^{s,*}_{r} \geq C_2 \cdot \|y\|^s_r \geq C_3 \cdot N^{r/m} \cdot \max_{j\in [N]} |c_j|
\end{equation}
For some constants $C_2,C_3 > 0$ depending only on $r$. By Eq.~\ref{eq:y}, we obtain:
\begin{equation}
\|y\|^{s,*}_{r} \geq \frac{\|y\|_{\infty}\cdot C_3}{C_0} \cdot N^{r/m}
\end{equation}
Then, for any $\beta > 0$, such that,
\begin{equation}
\gamma < \frac{\beta \cdot C_3}{C_0} \cdot N^{r/m} < 1
\end{equation}
we have: $[\rho \cdot U(X_{N}) \setminus \beta \cdot U(X_{N})] \subset \mathcal{W}^{\gamma}_{r,m}$. Hence, we have: 
\begin{equation}
\tilde{d}_N(\mathcal{W}^{\gamma}_{r,m}) \geq \tilde{b}_{N}(\mathcal{W}^{\gamma}_{r,m}) \geq \rho = C^{-1}_1 \cdot N^{-r/m}
\end{equation}
\end{proof}

\begin{lemma}\label{lem:calF}
Let $\sigma:\mathbb{R} \to \mathbb{R}$ be a piece-wise $C^1(\mathbb{R})$ activation function with $\sigma' \in BV(\mathbb{R})$. Let $\mathscr{f}$ be a class of neural networks with $\sigma$ activations. Let $\mathbb{Y} = \mathcal{W}_{r,m}$ and let $\mathcal{W}^{0,\infty}_{r,m} := \left\{f:[-1,1]^m\to \mathbb{R} \mid \textnormal{$f$ is $r$-smooth and } \|f\|^s_{r} < \infty \right\}$. Let $\mathcal{F}:\mathcal{W}^{0,\infty}_{r,m} \to \mathcal{W}^{0,\infty}_{r,m}$ be a continuous functional (w.r.t $\|\cdot\|^s_{r}$). Assume that for any $y \in \mathbb{Y}$ and $\alpha > 0$, if $y + \alpha \cdot \mathcal{F}(y)$ is non-constant, then it cannot be represented as a member of $\mathscr{f}$. Then, if $d(\mathscr{f};\mathbb{Y}) \leq \epsilon$, we have:
\begin{equation}
N_{\mathscr{f}} = \Omega(\epsilon^{-m/r})
\end{equation}
\end{lemma}

\begin{proof}
Let $\mathbb{Y}_1 = \mathcal{W}^{0.1,1.5}_{r,m} \subset \mathcal{W}^{0,1.5}_{r,m}$ (the selection of $\gamma = 0.1$ is arbitrary). We note that $\mathcal{F}(\mathbb{Y}_1)$ is a compact set as a continuous image of $\mathbb{Y}_1$. Since $\|\cdot \|^*_r$ is a continuous function over $\mathcal{F}(\mathbb{Y}_1)$ (w.r.t norm $\|\cdot\|^s_r$), it attains its maximal value $0\leq q < \infty$ within $\mathcal{F}(\mathbb{Y}_1)$. By the triangle inequality, for any $y \in \mathbb{Y}_1$, we have:
\begin{equation}\label{eq:0.1}
\|y + \epsilon \cdot \mathcal{F}(y)\|^*_r \geq \|y\|^*_r - \epsilon\cdot \|\mathcal{F}(y)\|^s_r \geq 0.1-\epsilon \cdot q  
\end{equation}
and also,
\begin{equation}\label{eq:yy'}
\forall y\in \mathbb{Y}_1: \|y-y'\|_{\infty} \leq \epsilon \cdot q
\end{equation}
We denote $\mathbb{Y}_2 := \{y+\alpha \cdot \mathcal{F}(y) \mid y \in \mathbb{Y}_1\}$. This is a compact set as a continuous image of the function $\mathcal{G}(y) := y+\epsilon \cdot \mathcal{F}(y)$, over the compact set $\mathbb{Y}_1$. In addition, for any constant $\epsilon < 0.1/q$, by Eq.~\ref{eq:0.1}, any $y \in \mathbb{Y}_2$ is a non-constant function. 


By Eq.~\ref{eq:yy'} and the triangle inequality, we have:
\begin{equation}
\forall y \in \mathbb{Y}_1: \|f(\cdot;\theta)-y'\|_{\infty} \leq 
\|f(\cdot;\theta)-y\|_{\infty} + \epsilon \cdot q
\end{equation}
Hence,
\begin{equation}
\sup_{y\in \mathbb{Y}_1} \inf_{\theta} \|f(\cdot;\theta)-y'\|_{\infty} \leq 
\sup_{y\in \mathbb{Y}_1} \inf_{\theta} \|f(\cdot;\theta)-y\|_{\infty} + \epsilon \cdot q = d(\mathscr{f};\mathbb{Y}_1) + \epsilon \cdot q
\end{equation}
In particular, 
\begin{equation}
d(\mathscr{f};\mathbb{Y}_2) = \sup_{y'\in \mathbb{Y}_2} \inf_{\theta} \|f(\cdot;\theta)-y'\|_{\infty} \leq 
d(\mathscr{f};\mathbb{Y}_1) + \epsilon \cdot q
\end{equation}
By the same argument, we can also show that $ 
d(\mathscr{f};\mathbb{Y}_1) \leq d(\mathscr{f};\mathbb{Y}_2) + \epsilon \cdot q$. 

By Lem.~\ref{lem:selection2}, there is a continuous selector $S:\mathbb{Y}_2 \to \Theta_{\mathscr{f}}$, such that,
\begin{equation}
\sup_{y' \in \mathbb{Y}_2}\| f(\cdot;S(y')) - y'\|_{\infty} \leq 2\sup_{y'\in \mathbb{Y}_2}\min_{\theta \in \Theta_{\mathscr{f}}}\| f(\cdot;\theta) - y'\|_{\infty} + \epsilon \leq 2(d(\mathscr{f};\mathbb{Y}_1) +\epsilon\cdot q)+ \epsilon 
\end{equation}
We note that $d(\mathscr{f};\mathbb{Y}_1) \leq 1.5 \cdot d(\mathscr{f};\mathbb{Y}) \leq 1.5 \epsilon$. Therefore, we have:
\begin{equation}
\sup_{y'\in \mathbb{Y}_2}\| f(\cdot;S(y')) - y'\|_{\infty} \leq (4+2q)\epsilon
\end{equation}
In particular, by defining $S(y) = S(y')$ for all $y \in \mathbb{Y}_2$, again by the triangle inequality, we have:
\begin{equation}
\tilde{d}(\mathscr{f};\mathbb{Y}_1) \leq \sup_{y\in \mathbb{Y}_1}\| f(\cdot;S(y)) - y\|_{\infty} \leq (4+2q)\epsilon + \epsilon \leq (5+2q)\epsilon
\end{equation}
By~\cite{devore}, we have:
\begin{equation}
(5+2q)\epsilon \geq \tilde{d}(\mathscr{f};\mathbb{Y}_1) \geq \tilde{d}_N(\mathbb{Y}_1) \geq C \cdot N^{-r/m}
\end{equation}
for some constant $C>0$ and $N = N_{\mathscr{f}}$. Therefore, we conclude that: $N_{\mathscr{f}} = \Omega(\epsilon^{-m/r})$. 
\end{proof}

We note that the definition of $\mathcal{F}(y)$ is very general. In the following theorem we choose $\mathcal{F}(y)$ to be the zero function. An alternative reasonable choice could $\mathcal{F}(y) := \frac{y}{2+y}$.


\generallower*

\proof{Follows immediate from Lem.~\ref{lem:calF} with $\mathcal{F}(y)$ being the zero function for all $y \in \mathbb{Y}$.}



\subsection{Proofs of Thms.~\ref{thm:3} and~\ref{thm:2}}

\begin{lemma}\label{lem:3}
Let $\sigma:\mathbb{R} \to \mathbb{R}$ be universal, piece-wise $C^1(\mathbb{R})$ activation function with $\sigma' \in BV(\mathbb{R})$. Let $\mathcal{E}_{\mathscr{e},\mathscr{q}}$ be an neural embedding method. Assume that $\|e\|^s_1 \leq \ell_1$ for every $e \in \mathscr{e}$ and $\mathscr{q}$ is a class of $\ell_2$-Lipschitz neural networks with $\sigma$ activations and bounded first layer $\|W^1_q\|_1 \leq c$. Let $\mathbb{Y} := \mathcal{W}_{1,m}$. Assume that any non-constant $y\in \mathbb{Y}$ cannot be represented as a neural network with $\sigma$ activations. If the embedding method achieves error $d(\mathcal{E}_{\mathscr{e},\mathscr{q}},\mathbb{Y}) \leq \epsilon$, then, the complexity of $\mathscr{q}$ is:
\begin{equation}
N_{\mathscr{q}} = \Omega\left(\epsilon^{-\min(m,2m_1) }\right)
\end{equation}
where the constant depends only on the parameters $c$, $\ell_1$, $\ell_2$, $m_1$ and $m_2$.
\end{lemma}

\begin{proof}
Assume that $N_{\mathscr{q}} = o(\epsilon^{-(m_1+m_2)})$. For every $y \in \mathbb{Y}$, we have:
\begin{equation}
\inf_{\theta_e,\theta_q} \Big\|y - q(x,e(I;\theta_e);\theta_q) \Big\|_{\infty} \leq \epsilon
\end{equation}
We denote by $k$ the output dimension of $\mathscr{e}$. Let $\sigma \circ W^1_q$ be the first layer of $q$. We consider that $W^1_q \in \mathbb{R}^{w_1 \times (m_1+k)}$, where $w_1$ is the size of the first layer of $q$. One can partition the layer into two parts: 
\begin{equation}
\sigma(W^{1}_{q} (x,e(x;\theta_e))) = \sigma (W^{1,1}_q x + W^{1,2}_q e(I;\theta_e))
\end{equation}
where $W^{1,1}_q \in \mathbb{R}^{w_1 \times m_1}$ and $W^{1,2}_q \in \mathbb{R}^{w_1 \times k}$. We divide into two cases.

\paragraph{Case 1} Assume that $w_1 = \Omega(\epsilon^{-m_1})$. Then, by the universality of $\sigma$, we can approximate the class of functions $\mathscr{e}$ with a class $\mathscr{d}$ of neural networks of size $\mathcal{O}(k \cdot \epsilon^{-m_2})$ with $\sigma$ activations. To show it, we can simply take $k$ neural networks of sizes $\mathcal{O}((\epsilon/\ell_1)^{-m_2}) = \mathcal{O}(\epsilon^{-m_2})$ to approximate the $i$'th coordinate of $e$ separately. By the triangle inequality, for all $y \in \mathbb{Y}$, we have:
\begin{equation}
\begin{aligned}
&\inf_{\theta_d,\theta_q} \Big\|y - q(x,d(I;\theta_d) ; \theta_q ) \Big\|_{\infty} \\
\leq & \inf_{\theta_e,\theta_d,\theta_q} \left\{ \Big\|y - q(x,e(I;\theta_e) ; \theta_q ) \Big\|_{\infty}  + \Big\|q(x,d(I;\theta_d) ; \theta_q) - q(x,e(I;\theta_e) ; \theta_q ) \Big\|_{\infty} \right\}\\
\leq & \sup_{y} \inf_{\theta_d} \Big\{ \Big\|y - q(x,e(I;\theta^*_e) ; \theta^*_q ) \Big\|_{\infty}  + \Big\|q(x,d(I;\theta_d) ; \theta^*_q) - q(x,e(I;\theta^*_e) ; \theta^*_q ) \Big\|_{\infty} \Big\}\\
\leq & \sup_{y} \inf_{\theta_d}  \Big\|q(x,d(I;\theta_d) ; \theta^*_q) - q(x,e(I;\theta^*_e) ; \theta^*_q ) \Big\|_{\infty} + \epsilon
\end{aligned}
\end{equation}
where $\theta^*_q,\theta^*_e$ are the minimizers of $\Big\|y - q(x,e(I;\theta_e) ; \theta_q ) \Big\|_{\infty}$. Next, by the Lipschitzness of $\mathscr{q}$, we have:
\begin{equation}
\begin{aligned}
\inf_{\theta_d}\Big\|q(x,d(I;\theta_d) ; \theta^*_q) - q(x,e(I;\theta^*_e) ; \theta^*_q ) \Big\|_{\infty} \leq \ell_2 \cdot \inf_{\theta_d} \Big\| d(I;\theta_d) - e(I;\theta^*_e) \Big\|_{\infty} \leq \ell_2 \cdot \epsilon
\end{aligned}
\end{equation}
In particular,
\begin{equation}
\begin{aligned}
\inf_{\theta_d,\theta_q} \Big\|y - q(x,d(I;\theta_d);\theta_q) \Big\|_{\infty} \leq (\ell_2+1) \cdot \epsilon
\end{aligned}
\end{equation}
By Thm.~\ref{thm:generallower} the size of the architecture $q(x,d(I;\theta_d);\theta_q)$ is $\Omega(\epsilon^{-m})$. Since $N_{\mathscr{q}} = o(\epsilon^{-(m_1+m_2)})$, we must have $k =  \Omega(\epsilon^{-m_1})$. Otherwise, the overall size of the neural network $q(x,d(I;\theta_d);\theta_q)$ is $o(\epsilon^{-(m_1+m_2)}) + \mathcal{O}(k \cdot \epsilon^{-m_2}) = o(\epsilon^{-m})$ in contradiction. Therefore, the size of $\mathscr{q}$ is at least $w_1 \cdot k = \Omega(\epsilon^{-2m_1})$.

\paragraph{Case 2} Assume that $w_1 = o(\epsilon^{-m_1})$. In this case we approximate the class $W^{1,2}_q \cdot \mathscr{e}$, where $W^{1,2}_q \in \mathbb{R}^{w_1 \times k}$, where $\|W^{1,2}_q\|_1 \leq c$. The approximation is done using a class $\mathscr{d}$ of neural networks of size $\mathcal{O}(w_1 \cdot \epsilon^{-m_2})$. By the same analysis of Case 1, we have:
\begin{equation}
\begin{aligned}
\inf_{\theta_d,\theta_q} \Big\|y - \tilde{q}(x,d(I;\theta_d);\theta_q) \Big\|_{\infty} \leq (\ell_2+1) \cdot \epsilon
\end{aligned}
\end{equation}
where $\tilde{q} = q'(W^{1,1}_q x + \textnormal{I} \cdot d(I;\theta_d))$ and $q'$ consists of the layers of $q$ excluding the first layer. We notice that $W^{1,1}_q x + \textnormal{I} \cdot d(I;\theta_d)$ can be represented as a matrix multiplication $M \cdot (x,d(I;\theta_d))$, where $M$ is a block diagonal matrix with blocks $W^{1,1}_q$ and $\textnormal{I}$. Therefore, we achieved a neural network that approximates $y$. However, the overall size of $q(x,d(I;\theta_d);\theta_q)$ is $o(\epsilon^{-(m_1+m_2)}) + \mathcal{O}(w_1 \cdot \epsilon^{-m_2}) = o(\epsilon^{-m})$ in contradiction.
\end{proof}

\begin{lemma}\label{lem:4}
Let $\sigma$ be a universal piece-wise $C^1(\mathbb{R})$ activation function with $\sigma' \in BV(\mathbb{R})$. Let neural embedding method $\mathcal{E}_{\mathscr{e},\mathscr{q}}$. Assume that $\|e\|^s_1 \leq \ell_1$ and the output dimension of $e$ is $k = \mathcal{O}(1)$ for every $e \in \mathscr{e}$. Assume that $\mathscr{q}$ is a class of $\ell_2$-Lipschitz neural networks with $\sigma$ activations. Let $\mathbb{Y} := \mathcal{W}_{1,m}$. Assume that any non-constant $y\in \mathbb{Y}$ cannot be represented as neural networks with $\sigma$ activations. If the embedding method achieves error $d(\mathcal{E}_{\mathscr{e},\mathscr{q}},\mathbb{Y}) \leq \epsilon$, then, the complexity of $\mathscr{q}$ is:
\begin{equation}
N_{\mathscr{q}} = \Omega\left(\epsilon^{-m}\right)
\end{equation}
where the constant depends only on the parameters $\ell_1$, $\ell_2$, $m_1$ and $m_2$.
\end{lemma}

\proof{Follows from the analysis in Case 1 of the proof of Lem.~\ref{lem:3}.}


\three*

\begin{proof}
First, we note that since $\sigma' \in BV(\mathbb{R})$, we have: $\|\sigma'\|_{\infty}<\infty$. In addition, $\sigma$ is piece-wise $C^1(\mathbb{R})$, and therefore, by combining the two, it is Lipschitz continuous as well. Let $e := e(I;\theta_e)$ and $q := q(x,z;\theta_{q})$ be members of $\mathscr{e}$ and $\mathscr{q}$ respectively. By Lems~\ref{lem:boundnet} and~\ref{lem:boundlip}, we have:
\begin{equation}
\|e\|_{\infty} = \sup_{I \in \mathcal{I}}\|e(I;\theta_e)\|_{1} \leq \ell_1 \cdot \|I\|_1 \leq m_2 \cdot \ell_1
\end{equation}
and also
\begin{equation}
\textnormal{Lip}(e) \leq \ell_1
\end{equation}
Since the functions $e$ are continuously differentiable, we have:
\begin{equation}
\sum_{1 \leq \vert \textbf{k} \vert_1 \leq 1} \|D^{\textbf{k}}e \|_{\infty} \leq  \|\nabla e\|_{\infty} \leq \textnormal{Lip}(e) \leq \ell_1
\end{equation}
Hence,
\begin{equation}
\|e\|^{s}_1 \leq  (m_2+1) \cdot \ell_1
\end{equation}
By similar considerations, we have: $\textnormal{Lip}(q) \leq \ell_2$. Therefore, by Lem.~\ref{lem:3}, we have the desired.
\end{proof}


\two*

\proof{Follows from Lem.~\ref{lem:4} and the proof of Thm.~\ref{thm:3}.}

\subsection{Proof of Thm.~\ref{thm:4}}

\begin{lemma}\label{lem:compactY}
Let $y \in \mathcal{W}_{r,m}$. Then, $\{y_I\}_{I \in \mathcal{I}}$ is compact and $F:I \mapsto y_I$ is a continuous function.
\end{lemma}

\begin{proof}
First, we note that the set $\mathcal{X} \times \mathcal{I} = [-1,1]^{m_1+m_2}$ is compact. Since $y$ is continuous, it is uniformly continuous over $\mathcal{X} \times \mathcal{I}$. Therefore, 
\begin{equation}
\begin{aligned}
\lim_{I \to I_0} \| y_I - y_{I_0} \|_{\infty} = \lim_{I \to I_0} \sup_{x \in \mathcal{X}} \|y(x,I) - y(x,I_0) \|_2 = 0
\end{aligned}
\end{equation}
In particular, the function $F:I \mapsto y_I$ is a continuous function. In addition, since $\mathcal{I} = [-1,1]^{m_2}$ is compact, the image $\{y_I\}_{I \in \mathcal{I}}$ of $F$ is compact as well.
\end{proof}



\begin{lemma}
Let $\sigma$ be a universal, piece-wise $C^1(\mathbb{R})$ activation function with $\sigma' \in BV(\mathbb{R})$ and $\sigma(0)=0$. Let $\hat{\mathbb{Y}} \subset \mathbb{Y} = \mathcal{W}_{r,m}$ be a compact set of functions $y$, such that, $y_I$ cannot be represented as a neural network with $\sigma$ activations, for any $I \in \mathcal{I}$. Then, there are classes $\mathscr{g}$ and $\mathscr{f}$ of neural networks with $\sigma$ and ReLU activations (resp.), such that, $d(\mathcal{H}_{\mathscr{f},\mathscr{g}};\hat{\mathbb{Y}})\leq \epsilon$ and $N_{\mathscr{g}} = \mathcal{O}\left(\epsilon^{-m_1/r}\right)$, where the constant depends on $m_1,m_2$ and $r$.
\end{lemma}

\begin{proof}
By the universality of $\sigma$, there is a class of neural networks $\mathscr{g}$ with $\sigma$ activations of size:
\begin{equation}
N_{\mathscr{g}} = \mathcal{O}\left(\epsilon^{-m_1/r}\right) 
\end{equation}
such that,
\begin{equation}
\forall p \in \mathcal{W}_{r,m_1}: \inf_{\theta_g\in \Theta_{\mathscr{g}}} \|g(\cdot ;\theta_g) - p \|_{\infty} \leq \epsilon
\end{equation}
Let $\mathbb{Y}' := \bigcup_{I\in \mathcal{I},y\in \hat{\mathbb{Y}}} \{y_I\}$. We note that, $\mathbb{Y}' \subset \mathcal{W}_{r,m_1}$. Therefore, 
\begin{equation}
\forall y\in \hat{\mathbb{Y}}~\forall I \in \mathcal{I}:  \inf_{\theta_g\in \Theta_{\mathscr{g}}} \|g(\cdot ;\theta_g) - y_I \|_{\infty} \leq \epsilon
\end{equation}
By Lem.~\ref{lem:selection2}, there is a continuous selector $S:\mathbb{Y}' \to \Theta_{\mathscr{g}}$, such that, for any $p \in \mathbb{Y}'$, we have:
\begin{equation}\label{eq:t1}
\begin{aligned}
\|g(\cdot ;S(p)) - p\|_{\infty} \leq 2\inf_{\theta_g\in \Theta_{\mathscr{g}}} \|g(\cdot ;\theta_g) - p \|_{\infty}+\epsilon \leq 3\epsilon
\end{aligned}
\end{equation}
We notice that the set $\mathcal{I}\times \hat{\mathbb{Y}}$ is compact as a product of two compact sets. Since $y_I$ is continuous with respect to both $(I,y) \in \mathcal{I}\times \hat{\mathbb{Y}}$, we can define a continuous function $S'(I,y) := S(y_I)$. Since $S'$ is continuous over a compact set, it is bounded as well. We denote by $\mathbb{B}$, a closed ball around $0$, in which the image of $S'$ lies. In addition, by the Heine-Cantor theorem, we have:
\begin{equation}
\begin{aligned}
\forall& \delta > 0~\exists \epsilon > 0 ~\forall I_1,I_2\in \mathcal{I},y_1,y_2\in \hat{\mathbb{Y}}:\\
&\|(I_1,y_1)-(I_2,y_2)\| \leq \delta \implies \|S'(I_1,y_1) - S'(I_2,y_2)\|_2 \leq \epsilon
\end{aligned}
\end{equation}
where the metric $\| \cdot \|$ is the product metric of $\mathcal{I}$ and $\hat{\mathbb{Y}}$. In particular, we have:
\begin{equation}
\begin{aligned}
\forall&~\delta > 0~\exists \epsilon > 0~\forall I_1,I_2\in \mathcal{I},y\in \hat{\mathbb{Y}}:\\
&\|I_1-I_2\|_2 \leq \delta \implies \|S'(I_1,y) - S'(I_2,y)\|_2 \leq \epsilon
\end{aligned}
\end{equation}
Therefore, since the functions $S'_y(I) := S'(I,y)$ (for any fixed $y$) are uniformly bounded and share the same rate of uniform continuity, by~\cite{hanin2017approximating}, for any $\hat{\epsilon}>0$, there is a large enough ReLU neural network $\mathscr{f}$, such that,
\begin{equation}
\sup_{y} \inf_{\theta_f \in \Theta_{\mathscr{f}}}\|S'_y(\cdot)-f(\cdot;\theta_f)\|_{\infty} \leq \hat{\epsilon}
\end{equation}
Since $g(x;\theta_g)$ is continuous over the compact domain, $\mathcal{X} \times \mathbb{B}$, by the Heine-Cantor theorem, $g$ is uniformly continuous. Hence, for any small enough $\hat{\epsilon}>0$, we have:
\begin{equation}\label{eq:t2}
\forall y \in \hat{\mathbb{Y}}: \inf_{\theta_{f}\in \Theta_{\mathscr{f}}}\sup_{I} \|g(\cdot ;f(I;\theta_f)) - g(\cdot ;S'_y(I)) \|_{\infty} \leq \epsilon
\end{equation}
In particular, by Eqs.~\ref{eq:t1} and~\ref{eq:t2} and the triangle inequality, we have the desired: 
\begin{equation}
\forall y \in \hat{\mathbb{Y}}~\forall I \in \mathcal{I}:  \inf_{\theta_f\in \Theta_{\mathscr{f}}}\sup_{I}\|g(\cdot ;f(I;\theta_f)) - y_I \|_{\infty} \leq 4\epsilon
\end{equation}
\end{proof}


\four*

\proof{Follows immediately for $\hat{\mathbb{Y}} = \{y\}$.}

\subsection{Proof of Thm.~\ref{thm:5}}


\five*

\begin{proof} We would like to approximate the function $S$ using a neural network $f$ of the specified complexity. Since $S \in \mathcal{P}_{r,w,c}$, we can represent $S$ in the following manner: 
\begin{equation}
S(I) = M \cdot P(I)
\end{equation}
Here, $P:\mathbb{R}^{m_2} \to \mathbb{R}^{w}$ and $M \in \mathbb{R}^{N_{\mathscr{g}} \times w}$ is some matrix of bounded norm $\|M\|_{1} \leq c$.  We recall that any constituent function $P_{i}$ are in $\mathcal{W}_{r,m_2}$. By~\cite{Mhaskar:1996:NNO:1362203.1362213}, such functions can be approximated by neural networks of sizes $\mathcal{O}(\epsilon^{-m_2/r})$ up to accuracy $\epsilon > 0$. Hence, we can approximate $S(I)$ using a neural network $f(I) := M \cdot H(I)$, where $H : \mathbb{R}^{m_2} \to \mathbb{R}^{w}$, such that, each coordinate $H_i$ is of size $\mathcal{O}(\epsilon^{-m_2/r})$. The error of $f$ in approximating $S$ is therefore upper bounded as follows:
\begin{equation}
\begin{aligned}
\|M \cdot H(I) - M \cdot P(I) \|_1 &\leq \|M\|_1 \cdot \|H(I) - P(I) \|_1 \\
&\leq c \cdot \sum^{w}_{i=1} \vert H_i(I) - P_i(I) \vert \\
&\leq c \cdot w \cdot \epsilon
\end{aligned}
\end{equation}
In addition,
\begin{equation}
\|M\cdot P(I)\|_1 \leq \|M\|_1 \cdot \|P(I)\|_1 \leq c \cdot w
\end{equation}
Therefore, each one of the output matrices and biases in $S(I)$ is of norm bounded by $c \cdot w$. 

Next, we denote by $W^i$ and $b^i$ the weight matrices and biases in $S(I)$ and by $V^i$ and $d^i$ the weight matrices and biases in $f(I)$. We would like to prove by induction that for any $x \in \mathcal{X}$ and $I \in \mathcal{I}$, the activations of $g(x;S(I))$ and $g(x;f(I))$ are at most $\mathcal{O}(\epsilon)$ distant from each other and the norm of these activations is $\mathcal{O}(1)$. 

{\bf Base case:} Let $x \in \mathcal{X}$. Since $\mathcal{X} = [-1,1]^{m_1}$, we have, $\|x\|_1 \leq m_1 =: \alpha^1$. In addition, we have:
\begin{equation}
\begin{aligned}
\|\sigma(W^1 x + b^1) - \sigma(V^1 x + d^1) \|_1 &\leq L \| (W^1 x + b^1) - (V^1 x + d^1) \|_1 \\
&\leq L\|W^1-V^1\|_1 \|x\|_1 +\|b^1-d^1\|_1 \\
&\leq m_1 \cdot L \cdot c \cdot w \cdot \epsilon + c \cdot w \cdot \epsilon \\
&=: \beta^1 \cdot \epsilon
\end{aligned}
\end{equation}
Here, $L$ is the Lipschitz constant of $\sigma$.

{\bf Induction step:} let $x_1$ and $x_2$ be the activations of $g(x;S(I))$ and $g(x;f(I))$ in the $i$'th layer. Assume that there are constants $\alpha^i,\beta^i > 0$ (independent of the size of $g$, $x_1$ and $x_2$), such that, $\|x_1 - x_2\|_1 \leq \beta^i \cdot \epsilon$ and $\|x_1\|_1 \leq \alpha^i$. Then, we have:
\begin{equation}\label{eq:alpha}
\begin{aligned}
\|\sigma(W^{i+1}x_1 + b^{i+1})\|_1 &= \|\sigma(W^{i+1}x_1 + b^{i+1}) - \sigma(0)\|_1 \\
&\leq L \cdot \|W^{i+1}x_1 + b^{i+1} - 0 \|_1 \\
&\leq L \cdot \|W^{i+1} x_1 \|_1 + L \cdot \|b^{i+1}\|_1\\
&\leq L \cdot \|W^{i+1} \|_1 \cdot \|x_1\|_1 + L \cdot c \cdot w\\
&\leq L \cdot c \cdot w (1 + \alpha^i) =: \alpha^{i+1}
\end{aligned}
\end{equation}
and also:
\begin{equation}\label{eq:beta}
\begin{aligned}
&\|\sigma(W^{i+1} \cdot x_1 + b^{i+1}) - \sigma(V^{i+1} x_2 + d^{i+1})\|_1 \\ 
\leq& L \cdot \| (W^{i+1} \cdot x_1 + b^{i+1}) - (V^{i+1} x_2 + d^{i+1}) \|_1 \\
\leq& L \cdot \| W^{i+1} x_1 - V^{i+1} x_2 \|_1 + L \cdot \|b^{i+1} - d^{i+1} \|_1 \\
\leq& L \cdot \| W^{i+1} x_1 - V^{i+1} x_2 \|_1 + L \cdot \epsilon \\
\leq& L \cdot (\| W^{i+1} \|_1 \cdot \|x_1 - x_2 \|_1 + \|W^{i+1} - V^{i+1}\|_1 \cdot \|x_2\|_1) + L \cdot \epsilon \\
\leq& L \cdot (c \cdot w \cdot \|x_1 - x_2 \|_1 + c\cdot w\cdot \epsilon \cdot \|x_2\|_1) + L \cdot \epsilon \\
\leq& L \cdot (c \cdot w \cdot \|x_1 - x_2 \|_1 + c\cdot w\cdot \epsilon \cdot (\|x_1\|_1 + \|x_1-x_2\|_1)) + L \cdot \epsilon \\
\leq& L \cdot (c \cdot w \cdot \beta^i \cdot \epsilon + c\cdot w\cdot \epsilon \cdot (\alpha^i + \beta^i \cdot \epsilon)) + L \cdot \epsilon \\
\leq& L (c \cdot w \cdot (2\beta^i +\alpha^i) + 1)\cdot \epsilon \\
=:& \beta^{i+1} \cdot \epsilon 
\end{aligned}
\end{equation}
If $i+1$ is the last layer, than the application of $\sigma$ is not present. In this case, $\alpha^{i+1}$ and $\beta^{i+1}$ are the same as in Eqs.~\ref{eq:alpha} and~\ref{eq:beta} except the multiplication by $L$. Therefore, we conclude that $\|g(\cdot;S(I))-g(x;f(I))\|_{\infty} = \mathcal{O}(\epsilon)$. 

Since $f$ consists of $w$ hidden functions $H_{i}$ and a matrix $M$ of size $w \cdot N_{\mathscr{g}}$, the total number of trainable parameters of $f$ is: $N_{\mathscr{f}} = \mathcal{O}(w^{1+m_2/r} \cdot \epsilon^{-m_2/r} + w \cdot N_{\mathscr{g}})$ as desired.
\end{proof}

\end{document}